\documentclass{article}[letterpaper]
\usepackage{twemojis}
\usepackage{amssymb}
\usepackage{microtype}
\usepackage{graphicx}
\usepackage{subcaption}
\usepackage{booktabs} 
\usepackage{algorithm}
\usepackage{algorithmic}
\usepackage{ifthen}
\newboolean{icml}
\setboolean{icml}{False}
\newboolean{incappendix}
\setboolean{incappendix}{True}
\newboolean{comments}
\setboolean{comments}{True}

\newcommand{\RETURN}{\STATE \textbf{return} }
\usepackage{hyperref}

\ifthenelse{\boolean{icml}}{\usepackage[accepted]{icml2026}}{
\usepackage{fullpage}
\usepackage{parskip}
\usepackage{natbib}
}

\usepackage{amsmath}
\usepackage{amssymb}
\usepackage{mathtools}
\usepackage{amsthm}

\definecolor{tobiComment}{RGB}{150, 10, 150}
\definecolor{fannyComment}{RGB}{150, 10, 250}
\definecolor{florianComment}{RGB}{100, 100, 250}
\definecolor{ChahineComment}{RGB}{100, 150, 150}
\definecolor{ZifanComment}{RGB}{200, 100, 100}
\newcommand{\noteby}[3]{{\colorbox{#2}{\bfseries\sffamily\scriptsize\textcolor{white}{#1}}}{\textcolor{#2}{\sf\small\textit{#3}}}}

\ifthenelse{\boolean{comments}}{

\newcommand{\fd}[1]{\noteby{Florian}{florianComment}{ #1}}
\newcommand{\fy}[1]{\noteby{Fanny}{fannyComment}{ #1}}
\newcommand{\cn}[1]{\noteby{Chahine}{ChahineComment}{ #1}}
\newcommand{\zf}[1]{\noteby{Zifan}{ZifanComment}{ #1}}}
{
\newcommand{\fd}[1]{}
\newcommand{\fy}[1]{}
\newcommand{\cn}[1]{}
\newcommand{\zf}[1]{}
}

\usepackage{soul,xcolor}
\setstcolor{red}

\usepackage[capitalize,noabbrev]{cleveref}
\Crefname{theorem}{Theorem}{Theorems}
\crefname{proposition}{Proposition}{Propositions}
\newtheoremstyle{thmstyle}
  {6pt} 
  {2pt} 
  {\itshape} 
  {} 
  {\bfseries} 
  {.} 
  {.5em} 
  {} 

\newtheoremstyle{defstyle}
  {6pt} 
  {2pt} 
  {} 
  {} 
  {\bfseries} 
  {.} 
  {.5em} 
  {} 

\theoremstyle{thmstyle}
\newtheorem{theorem}{Theorem}[section]
\newtheorem{proposition}[theorem]{Proposition}
\newtheorem{lemma}[theorem]{Lemma}
\newtheorem{corollary}[theorem]{Corollary}

\theoremstyle{defstyle}

\theoremstyle{remark}

\usepackage[textsize=tiny]{todonotes}

\DeclareMathOperator*{\EE}{\mathbb{E}}
\DeclareMathOperator*{\PP}{\mathbb{P}}
\newcommand{\istar}[0]{i^\star}
\newcommand{\Ihat}[0]{\widehat{I}}
\newcommand{\IhatSR}[0]{\widehat{I}_{\operatorname{SR}}}
\newcommand{\IhatUCBE}[0]{\widehat{I}_{\operatorname{UCB-E}}}

\newcommand{\prn}[1]{\left(#1\right)}
\newcommand{\mustar}[0]{\mu^\star}
\newcommand{\logbar}[0]{\overline{\log}}
\newcommand{\crl}[1]{\left\{ #1\right\}}
\renewcommand{\S}[0]{\mathbf{S}}

\newcommand{\cX}[0]{\mathcal{X}}
\newcommand{\Uniform}[0]{\operatorname{Uniform}}
\renewcommand{\r}[0]{\mathbf{r}}
\newcommand{\R}[0]{\mathbf{R}}
\newcommand{\sr}[0]{\boldsymbol{r}}
\newcommand{\sR}[0]{\boldsymbol{R}}
\newcommand{\Var}[0]{\operatorname{Var}}
\newcommand{\brk}[1]{\left[#1\right]}
\newcommand{\TV}[0]{\operatorname{TV}}
\newcommand{\BC}[0]{\operatorname{BC}}
\newcommand{\cD}[0]{\mathcal{D}}
\newcommand{\cZ}[0]{\mathcal{Z}}
\newcommand{\NN}[0]{\mathbb{N}}
\newcommand{\cJ}[0]{\mathcal{J}}
\newcommand{\Cov}[0]{\mathrm{Cov}}

\usepackage{pifont}
\usepackage{fontawesome5}
\newcommand{\IhatSc}{\widehat{I}_{\text{\ding{34}}}} 

\usepackage{minitoc}
\usepackage{tocloft}

\makeatletter
\newcommand\blfootnote[1]{%
  \begingroup
  \def\@thefnmark{}
  \renewcommand\@makefnmark{}
  \@footnotetext{#1}%
  \endgroup
}
\makeatother

\ifthenelse{\boolean{icml}}{\icmltitlerunning{Cutting LLM Evaluation Costs with SySRs: A Bandit Algorithm That Provably Exploits Model Similarity}
\makeatletter
\g@addto@macro\Notice@String{\newline \newline Code available at \href{https://github.com/zifanlyu/llm-bandits-sysrs/tree/main}{\faGithub\ \texttt{llm-bandits-sysrs}}. 
}
\makeatother
}{}

\begin{document}
\date{}
\ifthenelse{\boolean{icml}}{
\twocolumn[
  \icmltitle{Cutting LLM Evaluation Costs with SySRs: \\ A Bandit Algorithm That Provably Exploits Model Similarity}
  \begin{icmlauthorlist}
    \icmlauthor{Zifan Lyu}{ETH}
    \icmlauthor{Chahine Nejma}{CS,ENS}
    \icmlauthor{Tobias Wegel}{ETH}
    \icmlauthor{Fanny Yang}{ETH}
    \icmlauthor{Florian E. Dorner}{ETH,MPI}
  \end{icmlauthorlist}

  \icmlaffiliation{ETH}{ETH Zurich}
  \icmlaffiliation{MPI}{Max-Planck Institute for Intelligent Systems, Tübingen}
  \icmlaffiliation{CS}{Centrale Supélec}
  \icmlaffiliation{ENS}{Ecole Normale Supérieure de Cachan}
  \icmlcorrespondingauthor{Zifan Lyu}{zifan.lyu@inf.ethz.ch}
  \icmlcorrespondingauthor{Florian Dorner}{florian.dorner@tuebingen.mpg.de}

  \icmlkeywords{Machine Learning, ICML}

  \vskip 0.3in
]
\printAffiliationsAndNotice{} 
}{\title{Cutting LLM Evaluation Costs with SySRs \twemoji[scale=0.5]{scissors}: \\ A Bandit Algorithm that Provably Exploits Model Similarity}
\author{%
Zifan Lyu$^1$ \quad Chahine Nejma$^{2,3}$ \quad Tobias Wegel$^1$ \quad Fanny Yang$^1$ \quad Florian E. Dorner$^{1,4}$ \\[0.4em]
{\small $^1$ETH Zurich \quad $^2$Centrale Supélec \quad $^3$ENS de Cachan \quad $^4$MPI for Intelligent Systems, Tübingen} \\[0.2em]
{\small \begin{tabular}{cc} Correspondence to:
\texttt{zifan.lyu@inf.ethz.ch} & \texttt{florian.dorner@tuebingen.mpg.de}
\end{tabular}}%
}
\maketitle 
\blfootnote{Code available at \href{https://github.com/zifanlyu/llm-bandits-sysrs/tree/main}{\faGithub\ \texttt{llm-bandits-sysrs}}.}}

\begin{abstract}
\noindent
    Large Language Models are typically benchmarked by evaluating every model on every test query. For practitioners seeking the best model to deploy, this is often wasteful: if a model clearly performs worse than others, there is no need to precisely estimate its performance. Best-arm identification algorithms can be naturally applied to drastically reduce costs by adaptively allocating evaluation budget. Further, language models often respond similarly to the same prompt\textemdash a property previous work has tried to leverage with mixed success. We propose \underline{Sy}nchronized \underline{S}uccessive \underline{R}eject\underline{s} (SySRs), augmenting the classical Successive Rejects algorithm with paired comparisons. Unlike prior attempts to leverage model similarity in best-model identification, our approach is hyperparameter-free and enjoys performance guarantees that improve with the degree of similarity between evaluated models. Empirically, our method outperforms all baselines in terms of average error rate across 15 standard benchmarks, and in terms of worst-case budget for reliably identifying the best model. 
\end{abstract}

\addtocontents{toc}{\protect\setcounter{tocdepth}{-1}}
\section{Introduction}
\begin{figure*}[t]
\centering
\includegraphics[width=1.0\textwidth]{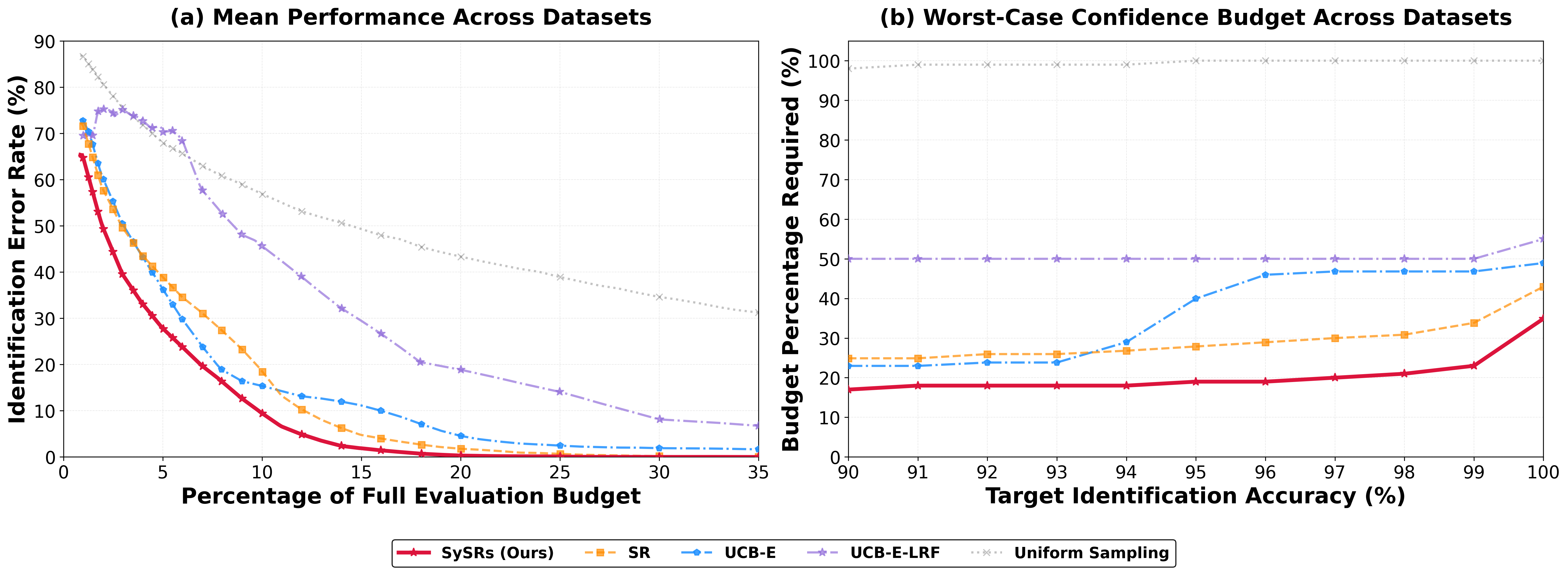}
\caption{\textbf{a)}: Error rate using a given percentage of model/query pairs, averaged over 15 LLM benchmarks. \textbf{b)} Percentage of model/query pairs required to identify the best model with a given accuracy, worst case over 15 LLM benchmarks. In both cases, baselines use the best-performing hyperparameter at each point on the $x$-axis. Our hyperparameter-free SySRs outperforms all baselines despite that.}
 \label{fig:mean}
\end{figure*}
    As Large Language Models become more and more capable, evaluating and comparing their capabilities becomes increasingly pressing and expensive. With models able to reliably solve a growing number of relevant real-world tasks, the question which model to deploy is more relevant than ever. At the same time, the costs of answering this question are at an all-time high: Test-time scaling and agent loops cause surging inference costs, while mounting task complexity increases the cost of gold standard evaluations of models' solutions. In particular, the labeling effort for generative tasks is much greater compared to earlier classification tasks, which could be evaluated by matching all models' outputs to fixed ground-truth labels. For example, modern benchmark tasks like theorem proving or software development require independent evaluation of every single model's response to a query, greatly increasing annotation costs.

    Attempts to reduce evaluation costs include replacing expert annotators with Large Language Models \citep{zheng2023judging} or focusing evaluation on a smaller, pre-selected set of particularly informative subtasks \citep{polo2024tinybenchmarks}.  Unfortunately, both approaches lack reliability and can produce severely misleading evaluation scores and model rankings \citep{bavaresco2025llms,dorner2024limits,yauney2026,zhang2025how}. Focusing on model selection, \citet{zhou2025speeding} propose an orthogonal angle with strong statistical guarantees. They treat model evaluation as a multi-armed bandit problem 
    in which models represent arms and per-query accuracy is the reward.
    Instead of fully evaluating all models, best-arm algorithms such as UCB-E can then decide which model to evaluate in an online fashion. Importantly, if some models 
    turn out to be clearly worse than others, the algorithm can discard them early on, 
    and instead focus on distinguishing between the best models. Indeed, \citet{zhou2025speeding} demonstrate that these algorithms can starkly reduce the cost of model selection.

    While the theoretical properties of generic best-arm identification are well-understood
    \citep{audibert2010best,kaufmann2016complexity,ariu2021policy,qin2022open}, the LLM benchmarking setting offers more structure than what is usually considered in the theoretical literature. In particular, across tasks, models' performances are often highly correlated \citep{mania2019model,kim2025correlated}: Particularly easy tasks will be solved by the majority of models, while most models will fail to solve particularly hard tasks. Efficient evaluation algorithms should take this \emph{model similarity} into account. However, the literature on multi-armed bandits with correlated arms is scarce. While \citet{gupta2021best} discuss the best-arm identification setting, their approach is not directly applicable to LLM evaluation as it requires prior knowledge about the specific degree of correlation strength. Specific to LLM evaluation, \citet{zhou2025speeding} also proposed UCB-E-LRF -- a variant of the UCB-E algorithm \citep{audibert2010best} that uses low-rank factorization to predict model performance on unobserved evaluation cases. However, their approach comes without theoretical guarantees and does not consistently outperform standard UCB-E.

    \paragraph{Contributions.} In this work, we propose a novel algorithm called  \emph{\underline{Sy}nchronized \underline{S}uccessive \underline{R}eject\underline{s}} (SySRs), based on the classical Successive Rejects algorithm \citep{audibert2010best}. Successive Rejects proceeds in stages: at each stage, it estimates each remaining arm's performance and eliminates the worst-performing one. SySRs exploits the additional structure of benchmarks compared to the classical bandit setting via \textit{synchronization}: Rather than independently sampling test queries for each model, each stage of SySRs samples a subset of test queries and evaluates all remaining models on \emph{the same} set. We prove statistical performance guarantees for SySRs that scale with the correlation between evaluated models. \looseness=-1Empirically,  we find that averaged over 15 standard LLM benchmarks, SySRs outperforms all algorithms from prior work, including \citet{zhou2025speeding} (see \cref{fig:mean}a). In addition, the performance of SySRs is highly consistent: As shown in \cref{fig:mean}b), it reliably identifies the best model on all 15 benchmarks using at most 35\% of the respective model/test-point pairs, while all other methods require more data on at least one benchmark to achieve the same.
    In short, our contributions are:
    \begin{itemize}
        \item We introduce a novel correlation-aware best-arm identification algorithm called \emph{\underline{Sy}nchronized \underline{S}uccessive \underline{R}eject\underline{s}} (SySRs) (\cref{sec:CSR}).
        \item We prove an upper bound on the error rate of SySRs using a novel hardness measure that improves with between-model correlation (\cref{sec:guarantees}).
        \item We run extensive experiments on 15 standard LLM benchmarks, demonstrating SySRs's strong performance compared to existing algorithms (\cref{sec:experiments}).
    \end{itemize}

\section{Related Work}
\label{sec:related-work}
\paragraph{Efficient LLM Evaluation.}
Perhaps the most common approach to reducing annotation costs in LLM evaluation is to replace expert annotators with so-called LLM judges \citep{zheng2023judging}. While this approach can greatly reduce annotation costs and often works well, it has no bearing on inference costs and lacks reliability \citep{dorner2024limits,bavaresco2025llms} when applied na\"ively. One explanation for these reliability issues is model similarity \citep{mania2019model} 
such as correlated error patterns between different models \citep{kim2025correlated}: A judge model that is too similar to the model it is supposed to evaluate might be unable to identify the latter's mistakes \citep{goel2025great}.

While model similarity undermines LLM judges, more recent works aim to instead explicitly make use of it. Specifically, they aim to exploit similarity to construct smaller test sets that remain predictive of overall model performance \citep{vivek2024anchor,polo2024tinybenchmarks}. By reducing the amount of evaluated data points, these \textit{subset selection} methods cut down on both annotation and inference costs. However, they lack theoretical guarantees and do not reliably outperform random subsampling test sets \citep{yauney2026,zhang2025how}. Moreover, subset selection requires access to large amounts of historical evaluation data on a given benchmark to obtain the downsized test set.
In contrast, our proposed SySRs algorithm can be applied to any benchmark, including novel ones with no historical data. In fact, SySRs is complementary to subset selection: whenever evaluation on a smaller test set is known to provide reliable information about overall benchmark performance, our approach can be applied to the subset to further speed up model selection.

\paragraph{Best-Model Identification.}
While benchmark \emph{scores} are often framed as absolute measurements of model capabilities \citep{suhr2025stop}, they are highly sensitive to seemingly irrelevant details like prompt wording \citep{sclar2023quantifying} or the precise degree of data cleaning \citep{recht2019imagenet}. Meanwhile, model \emph{rankings}\textemdash describing which models outperform others\textemdash are surprisingly stable across contexts \citep{miller2021accuracy,ruan2024observational}, granting them a larger degree of external validity \citep{hardt2026emerging}. As such, we believe that often the appropriate goal for efficient LLM evaluation is to accurately rank models rather than precisely estimate model performance. Among potential ranking objectives, we focus on top-1 identification (that is, identifying the best model), which is arguably the most relevant for practice as it directly maps to deployment decisions. 

\citet{zhou2025speeding} recognized that top-1 identification maps to the best-arm identification or \textit{pure exploration} setting from the multi-armed bandit literature, and suggested using the popular UCB-E algorithm for model selection. 
While general best-arm identification is widely studied \citep{bubeck2009pure,audibert2010best,kaufmann2016complexity,ariu2021policy},  there is limited prior work on multi-armed bandits with correlated arms. Most relevantly, \citet{gupta2020correlated,gupta2021multi,gupta2021best} study regret minimization and best-arm identification, modeling between-arm correlation as a known upper bound on one arm's expected reward conditioned on observing another arm's reward. However, in many cases no such nontrivial upper bound might be known \textit{a priori}, despite between-arm correlation.

Specific to the LLM evaluation setting, \citet{zhou2025speeding} also propose to exploit between-model correlation by combining UCB-E with low-rank factorization to predict LLM scores on unseen queries. However, their method does not come with any theoretical performance guarantees and underperforms UCB-E in our experiments. In contrast, our SySRs algorithm outperforms UCB-E on average and comes with theoretical guarantees that scale with the degree of correlation between models, while not requiring \textit{a priori} knowledge about that degree.

\section{Formal Setup and Existing Algorithms}
We consider the problem of identifying the best model among a set of $K$ candidate Large Language Models (LLMs) $f_1,\dots,f_K$ on a given test set $\mathcal{X} = \{x_1, \dots, x_L\}$, using a fixed finite testing budget of $n$ evaluations. For each prompt $x_i\in \mathcal{X}$, we denote the performance of model $f_j$ on $x_i$ as $S_{ij}$. As both the model's answer and the scoring can be non-deterministic, we model $S_{ij}$ to be independently drawn from a distribution $\phi_{ij}$ with bounded support in $[0,1]$. The  stochastic performance matrix $\S = (S_{ij})_{ij} \in [0,1]^{K \times L}$ then generalizes the deterministic performance matrix from \citet{zhou2025speeding}. 

Following \citet{zhou2025speeding}, we view adaptive LLM evaluation as revealing entries of the unknown matrix $\S$, but allow these entries to be stochastic.
In particular, we consider (potentially randomized) algorithms that operate in $n$ steps, and at each step $t \in \{1, \dots, n\}$
\begin{enumerate}
    \item Select a model index $i_t \in \{1, \dots, K\}$ and a test instance index $j_t \in \{1, \dots, L\}$.
    \item Observe the score $S_{i_t j_t} $. 
\end{enumerate}

The goal of the algorithm is to identify the model $f_i$, or ``arm'', with the largest average performance, defined as
\[
i^\star \in \arg\max_{i\in\{1,\dots,K\}} \mu_i \quad \text{where} \quad \mu_i := \frac{1}{L} \sum_{j=1}^L \mathbb{E}_{S_{ij} \sim \phi_{ij}} [S_{ij}]. 
\]
After $n$ steps, the algorithm produces a guess $\Ihat$ for the best arm $\istar$, and we measure an algorithm's performance in terms of its error probability $ e_n = \mathbb{P}(\Ihat \neq \istar)$. This probability is defined over both the noisy model responses and the algorithm's internal randomness, which determines the sequence of queried model-prompt pairs. 

Standard best-arm identification algorithms with fixed budget can straightforwardly be applied to this setting by viewing $\mu_i$ as the expectation $\EE[R_i]$ of the reward $R_i=S_{iJ}$, where a test set index $J$ is drawn uniformly from $\{1, \dots, L\}$. 
However, the matrix formulation allows us to additionally capture the correlation structure between models. Formally, indices $J$ sampled uniformly from $\{1, \dots, L\}$ induce random vectors $(R_1,...,R_K) = (S_{1J},...,S_{KJ})$ with the desired expectation $(\mu_1,...,\mu_K)$ and potentially non-trivial correlation structure. This contrasts 
with standard multi-armed bandit settings, where independence is typically assumed both across arms and tasks.

\subsection{Best-arm Identification Algorithms}\label{sec:classical}
In this section, we present how to use two classical and one tailored best-arm identification algorithms for LLM evaluation.
The \emph{Successive Rejects algorithm (SR)}, first introduced by~\citet{audibert2010best}, addresses the problem of finding the best-performing model by dividing the evaluation budget $n$ into $K-1$ phases and progressively eliminating models. The lengths of each phase follow a carefully designed hyperparameter-free schedule. Within each phase, all models in the active set are tested equally often by selecting one test instance for each model independently and uniformly at random. At the end of each phase, the single worst-performing model (based on empirical mean) is removed from the set of remaining models. After the last phase, exactly one model remains, which is returned as $\IhatSR$. 

\citet{audibert2010best} also introduced the  \emph{Upper Confidence Bound Exploration (UCB-E)} algorithm, a variant of the classical UCB algorithm \citep{auer2002using} that allows for the degree of exploration to be controlled by a hyperparameter $a$. The algorithm selects models based on an exploration index:
$    \hat{\mu}_i + \sqrt{a/s_i}
$ 
where $s_i$ is the number of times the model $i$ has been queried.
The algorithm runs for $n$ rounds, in each of which the model with the highest exploration index is evaluated by drawing a test instance independently and uniformly at random. 
After the last round, UCB-E returns the model $\IhatUCBE$ with the largest empirical mean.

As these classical algorithms do not make use of the given matrix structure, \citet{zhou2025speeding} propose two adaptations to the UCB-E tailored to the LLM evaluation setting: The first method is a simple modification of UCB-E that avoids evaluating the same model/query pair twice. Instead of sampling queries $x_j \in \mathcal{X}$ with replacement, the algorithm keeps track of which entries $x_j$ a model $f_i$ has already been evaluated on and ensures no model is evaluated twice on the same query. 
The second method called UCB-E-LRF is more involved. For deterministic scores $S_ij$ the algorithm fits an ensemble of low-rank matrices $\hat{\mathbf{S}}$ at each step, aiming to reconstruct the score matrix $\mathbf{S}$  from the limited already observed entries. It then selects \emph{both} the model $f_i$ to evaluate and the query $x_j$ to evaluate on, based on a combination of the observed entries, and both the mean and variance of the ensemble $\hat{\mathbf{S}}$. Both algorithms are described in detail in \cref{sec:ucb-e,sec:ucb-e-lrf}.
\subsection{Guarantees for Classical Best-Arm Identification}
\label{subsec:guarantees-classical}
Classical algorithms for Best-Arm Identification come with strong guarantees on the error probability $e_n$. 
The guarantees depend on a \emph{hardness measure} that captures the difficulty of the problem. Since distinguishing between two almost equivalent models is harder than separating the best candidate from a low performer, hardness measures usually depend on the gap between the mean scores $\mu_i$.

Formally, standard hardness measures are usually defined by first reordering the models/arms $i$ in terms of decreasing mean so that 
$ \mu_{(1)}\geq \mu_{(2)}\geq\cdots \geq \mu_{(K)}$. For $i > 1$, we then define the sub-optimality gaps
\[
\Delta_i = \mustar-\mu_{(i)},
\quad \text{where}\quad  \mustar=\max_{i\in{\{1,\dots,K\}}}\mu_i = \mu_{(1)} = \mu_{i^\star},
\]
setting $\Delta_1 = \Delta_2$ such that
\[
\Delta_{1} = \Delta_{2} 
\le \dots \le \Delta_{K}.
\]
In the classical fixed-budget best-arm identification setting, the hardness of the problem instance is typically captured by the quantities $H_1$ and $H_2$, defined by \citet{audibert2010best}:
\[
H_1 = \sum_{i=1}^K \frac{1}{\Delta_{i}^2}, \qquad 
H_2 = \max_{i\in\{2,\dots,K\}} \frac{i}{\Delta_{i}^2}.
\]
The two quantities are related by the inequalities
\[
H_2 \leq H_1 \leq \log(2K)\, H_2.
\]

Assuming the exploration parameter $a$ for the UCB-E algorithm is chosen to satisfy $0<a\leq \frac{25}{36}\frac{n-K}{H_1}$, \citet{audibert2010best} prove the following guarantee for UCB-E:
\begin{equation}
\label{eqn:guarantees-UCBE}
\begin{aligned}
    \PP\prn{\IhatUCBE \neq \istar} & \leq 2nK\exp\prn{-\frac{2a}{25}},
\end{aligned}
\end{equation} 
which \citet{zhou2025speeding} show to remain valid for LLM evaluation with queries sampled without replacement, when the scores $S_{ij}$ are deterministic.  
However, $H_1$ is usually not known \textit{a priori}, such that users face the difficult task of guessing the exploration parameter $a$ both small enough to achieve the guarantee and large enough to obtain a sufficiently small value in the exponential. Meanwhile, SR enjoys a similar guarantee \citep{audibert2010best}, without hyperparameters : \begin{equation}
\label{eqn:guarantees-SR}
\begin{aligned}
    \PP\prn{\IhatSR \neq \istar} &\leq \frac{K(K-1)}{2}\exp\prn{-\frac{n-K}{\logbar(K)H_2}}, \\
\end{aligned}
\end{equation}
 where  $\logbar(K) = \frac{1}{2}+\sum_{i=2}^K \frac{1}{i}$. 
Unlike UCB-E-LRF, neither of the two algorithms attempt to take advantage of the correlation between arms, which is reflected in neither bound depending on said correlation. On the other hand, due to its relatively high complexity, there are no known performance bounds for UCB-E-LRF. In the next section, we propose a novel algorithm that comes with performance guarantees that improve with the correlation between evaluated models. 

\section{Synchronized Successive Rejects (SySRs)}
\label{sec:CSR}
In this section, we introduce the \emph{\underline{Sy}nchronized \underline{S}uccessive \underline{R}eject\underline{s}} (SySRs) algorithm, inspired by Successive Rejects (SR). SR is particularly suited for adaptation to our setting for multiple reasons: 
First, one can easily modify it to utilize model correlation by evaluating different LLMs on the same test instances.
Second, unlike UCB-style algorithms that involve switching models at every step, SR evaluates active models on blocks of test instances in every phase, allowing for straightforward batched LLM inference. 
Lastly, SR is hyperparameter free, allowing us to avoid parameter tuning and adapt to correlation without \textit{a priori} knowledge of its strength. 
In particular, as we will demonstrate, adapting SR yields a clean theoretical analysis of the error probability that scales with the correlation strength.

Similar to SR, SySRs relies on a static allocation strategy. It splits the evaluation budget $n$ into $K-1$ unequally sized phases (with $K$ the number of models). 
The phases are carefully scheduled such that the budget is efficiently utilized. In particular, for $n>K$, the budget assigned to each phase $k=1,\dots,K-1$ is defined as $n_k-n_{k-1}$ per remaining model, where in turn $n_k$ is defined as
\begin{equation}
\label{eqn:budget-schedule}
     n_k = \Big\lceil \frac{1}{\overline{\log}(K)}\frac{n-K}{K+1-k}\Big\rceil,\quad \logbar(K) = \frac{1}{2}+\sum_{i=2}^K \frac{1}{i}.
\end{equation}
In phase $k$, all $K+1-k$ remaining models are evaluated \emph{on the same subset $\widetilde{\mathcal{J}}_k$ of $n_k-n_{k-1}$ queries}.
We sample the evaluation queries $\widetilde{\mathcal{J}}_k$ uniformly at random, \emph{without replacement} as \citet{zhou2025speeding} do for UCB-E to ensure we do not evaluate models on the same test instance more than once.

We then eliminate the worst model according to the empirical mean of rewards $\widehat{\mu}_{i,n_k} = \frac{1}{n_k} \sum_{j\in \mathcal{J}_k} S_{ij} $, where $\mathcal{J}_k = \bigcup_{t\leq k}  \widetilde{\mathcal{J}}_t$ is the set of queries the remaining models have been evaluated on. 
By design, after phase $k$, each remaining model has been evaluated exactly $n_k$ times. Moreover as verified by \citet{audibert2010best}, the total budget spent does not exceed the allowed budget $n$.  We summarize this procedure in \cref{alg:csr}.

\begin{algorithm}[h]
\caption{Synchronized Successive Rejects (SySRs)}
\label{alg:csr}
\begin{algorithmic}[1]
\REQUIRE Budget $n$, evaluation oracle that samples entries in score matrix $\S \in [0,1]^{K\times L}$
\STATE $\overline{\log}(K) \gets \frac{1}{2}+\sum_{i=2}^{K}\frac{1}{i}$ and $n_0 \gets 0$ 
\STATE $A_1 \gets \{1,\dots,K\}$, $\cJ \gets \crl{1,\ldots, L}$, and $\mathcal{J}_0 \gets \varnothing$
\FOR{$k=1$ to $K-1$} 
    \STATE $n_k \gets \left\lceil \frac{1}{\overline{\log}(K)}\frac{n-K}{K+1-k}\right\rceil$
    \STATE Draw $\widetilde{\mathcal{J}}_k \subseteq \cJ\setminus \mathcal{J}_{k-1}$ uniformly without replacement from all subsets such that $|\widetilde{\mathcal{J}}_k| = n_k-n_{k-1}$
    \STATE $\mathcal{J}_k \gets \mathcal{J}_{k-1} \cup \widetilde{\mathcal{J}}_k$
    \FOR{each arm $i\in A_k$}
        \STATE Sample $S_{ij}\sim \phi_{ij}$ for all $j\in \widetilde{\mathcal{J}}_k$ 
        \STATE $\widehat{\mu}_{i,n_k} \gets \frac{1}{n_k}\sum_{j\in \mathcal{J}_k} S_{ij}$
    \ENDFOR
    \STATE $\ell_k \in \arg\min_{i\in A_k}\widehat{\mu}_{i,n_k}$
    \STATE $A_{k+1} \gets A_k \setminus \{\ell_k\}$
\ENDFOR
\RETURN the last remaining element in $\IhatSc \in A_{K}$
\end{algorithmic}
\end{algorithm}

Contrary to the classical SR algorithm, which samples rewards independently for each arm, SySRs explicitly exploits the correlation between different models by evaluating the entire set of remaining models on the \emph{same} test instances in each phase. As with paired comparison tests, given between-model correlation, this allows for more precise estimates of the performance gap between two models. In the next section, we explain this in more detail and quantify the resulting gains from synchronization. 

\section{Guarantees for SySRs}\label{sec:guarantees}
To bound the error probability of SySRs, we must first capture the relationship between the models. To that end, we introduce a correlation-aware hardness measure. 
Analogous to $H_2$ from \cref{subsec:guarantees-classical}, we define 
\begin{equation}\label{eq:H3}
H_3
\;=\;
\max_{i\in\{2,\dots,K\}}
 i H_{3,i}\
\end{equation}
where we replace the arm-specific hardness $1/\Delta_i^2$ in $H_2$ by the new $H_{3,i}$, which we define as
\begin{equation}\label{eq:H3k}
H_{3,i} \;=\; \frac{2 V_{\star,i} + \frac{2}{3}(1+\Delta_i)\Delta_i}{\Delta_i^2}.
\end{equation}
Here, we denote the variance of the score difference between the optimal arm and the $i$-th sorted arm $(i)$ as 
\begin{align}
V_{\star,i} &= \Var(S_{\istar J} - S_{(i)J}), \label{eq:variance}
\\& = \Var(S_{\istar J}) + \Var(S_{(i)J}) - 2 \Cov (S_{\istar J}, S_{(i) J}) \nonumber
\end{align}
where indices $J\sim \text{Unif}(1,\cdots, L)$ are sampled uniformly and variance is taken jointly across the sampled indices and the stochasticity of the scores $S_{ij}\sim \phi_{ij}$.
Then we obtain the following upper bound for the error probability of SySRs. 

\begin{theorem}
\label{thm:csr_bound}
Let $\IhatSc$ be the arm chosen by SySRs (\cref{alg:csr}) with a budget of $n$. The probability $e_n$ that $\IhatSc$ is not the best arm $i^\star$ is then upper bounded as follows:
\[
e_n := \mathbb{P}(\IhatSc \neq i^\star) \;\le\; \frac{K(K-1)}{2} \exp\left( - \frac{n-K}{\overline{\log}(K) H_3} \right), 
\]
where $H_3$ is the correlation-aware hardness quantity defined in \cref{eq:H3}.
\end{theorem}
We prove \cref{thm:csr_bound} in \cref{sec:proofs}. Our proof follows  the classical SR analysis  in \citet{audibert2010best}. The key difference is that it uses a version of Bernstein's inequality 
rather than Hoeffding's inequality to bound the error rate for comparing the means of any pair of arms. 

Intuitively, SySRs reduces error probability by the same mechanism a paired
difference test improves compared to a two-sample test.
At the end of a given phase $k$, the probability of an error is upper bounded
by the probability that the estimated gap $\widehat{\mu}_{i^\star} - \widehat{\mu}_{(i)} $ is negative for 
any remaining model $(i)$. This probability in turn depends on the
variance $V_{\star,i}$ of the score difference $S_{i^\star j} - S_{(i) j}  $.  Formally, by the definition of
$H_3$ in \cref{eq:H3k}, the bound in \cref{thm:csr_bound} tightens
whenever the variance $V_{\star,i}$ is small. By exposing all models to the same test
instances, we make use of between-model correlation to reduce that variance (\cref{eq:variance}).
As discussed in detail later, this tightening is particularly decisive in the ''hard'' regime where the score gaps $\Delta_i$ are small and identifying the optimal arm requires the most samples. 

\subsection{Regimes Where Synchronization Pays Off}
While \cref{thm:csr_bound} and \cref{eqn:guarantees-SR} share a similar structure, we highlight the gains in our result by discussing some sufficient conditions for our bound to be (much) smaller. We define $\theta_i \in [0,1]$ as the normalized correlation parameter between the best arm and the $i$-th-ranked arm:
\[
 \theta_i := \frac{\Cov (S_{\istar J}, S_{(i) J})}{(\mu^\star-\Delta_i)(1-\mu^\star) }
\]
where $(\mu-\Delta)(1-\mu)$ is the maximum achievable covariance between two random variables in $[0,1]$ with means $\mu$ and $\mu-\Delta$ respectively. We focus on the case of positive between-model correlation, in which we have that $\theta_i \in [0,1]$.
In Proposition \ref{thm:theta-star}, we prove a detailed sufficient condition on $(\theta_i)_{i \neq 1}$ for $H_3$ to be smaller than $H_2$. The threshold $\theta^\star$ derived there tightly characterizes the boundary for improvement. However, due to the algebraic complexity of this result, we distill the following corollary. 

\begin{corollary}[Simplified Sufficient Conditions]
\label[corollary]{cor:sufficient_conditions}
If for every suboptimal arm, either
\[
    \theta_i > 2\Delta_i
    \qquad \text{or} \qquad
    \Delta_i < \tfrac{3}{2}\mu^{\star} - 1,
\]
then the correlation-aware hardness satisfies $H_3 < H_2$.
\end{corollary}

\Cref{cor:sufficient_conditions} implies that $H_3$ is commonly smaller than $H_2$, even when between-model correlation is modest. If all evaluated models have similar capabilities, all $\Delta_i$ are small such that small values of $\theta_i$ already suffice for the first condition. Similarly, if the best model has high average score $\mustar \approx 1$ while all other models perform noticeably better than chance ($\mu_i \gg 0.5$), the second condition is fulfilled even when there is no correlation.

We would like to highlight that the comparison of $H_2$ and $H_3$ relies on two separate mechanisms: a better suited analysis and an improved algorithm. Our analysis can equally be applied to standard SR, leading to an error bound where $H_3$ is replaced by 
\begin{equation*}
\begin{split}
    H'_3 &=  \max_{i\in\{2,\dots,K\}} i H'_{3,i}, \quad\text{where}\\
    H'_{3,i} &= \frac{2 \Var(S_{i^\star J}) + 2\Var(S_{(i)J}) + \frac{2}{3}(1+\Delta_i)\Delta_i}{\Delta_i^2}
\end{split}
\end{equation*}
uses the sum of per-arm variances rather than $V_{\star,i}$. In certain regimes (e.g., when the second condition of \cref{cor:sufficient_conditions} holds), this will yield tighter performance bounds than \cref{eqn:guarantees-SR}, even for standard SR. On the algorithmic side, SySRs improves performance by being able to leverage positive correlation between models. This is reflected in the scaling of $H_3$ with the correlation strength $\theta_i$, 
while $H'_3$ is unaffected by any between-arm correlation. 

Averaging over the datasets summarized in \cref{tab:dataset-diversity}, we indeed observe that usually $H_3 < H'_3 < H_2$. This shows that both tighter analysis and algorithmic improvements play an important role for our improved theoretical guarantees.

\subsection{The Two-Armed Case}
To build further intuition for the scaling of the correlation-aware hardness $H_3$, we focus on the simplified two-arm Bernoulli setting with means $\mu$ and $\mu - \Delta$. We discuss how $H_3$ scales with the performance gap $\Delta$ and correlation, and compare it to $H_2$. We consider a deterministic score matrix $\mathbf{S} \in \{0,1\}^{2 \times L}$ with the normalized correlation parameter 
$
    \theta = \Cov_{J\sim \text{Unif}(1,\cdots, L)}(S_{i^\star J}, S_{(2)J})/((\mu-\Delta)(1-\mu))
$.

In particular, we have the two following edge cases.

\noindent\emph{(i) Uncorrelated rewards ($\theta=0$).}
When arms are uncorrelated, the variance of the score difference equals the sum of the per-arm variances:
\begin{equation}\label{eq:Nocov}
    V_{\star,2} \;=\; \Var(S_{i^\star J} - S_{(2)J}) \;=\; \Var(S_{i^\star J}) + \Var(S_{(2)J}).
\end{equation}
As a result, in this regime, $H_3 = H'_3 = \Theta(1/\Delta^2)$ as $\Delta \to 0$, recovering the classical hardness $H_2$ up to constant factors.

\noindent\emph{(ii) Maximally correlated rewards ($\theta=1$).}
Under binary scores and deterministic per-instance scores,
this case arises naturally from a point-wise dominant model \citep{mania2020classifier}, where the optimal model is necessarily correct whenever the suboptimal model is correct ($S_{(2)j}=1 \implies S_{i^\star j}=1$). This structural dependence forces maximal correlation ($\theta=1$), leading to a simplified variance
\begin{equation}
\label{eqn:Bernstein-condition-analogue}
    V_{\star,2} = \Delta -\Delta^2
\end{equation} 
such that hardness scales as
\(
H_3 = \Theta\left(1/\Delta\right) 
\) for $\Delta\to 0$.

This improvement is reminiscent of fast rate phenomena in statistical learning theory. 
There, Bernstein-type conditions \citep{bartlett2006empirical} control the variance of the excess loss by its expectation, allowing sample complexities to interpolate between the slow rate $1/\varepsilon^2$ and the fast rate $1/\varepsilon$ \citep[Appendix B.5]{shalev2014understanding}. 
Here, \cref{eqn:Bernstein-condition-analogue} yields an analogous relation:
\begin{align*}
    \Var(S_{i^\star J}-S_{(2)J})
    &\leq \Delta = \EE[S_{i^\star J}-S_{(2)J}].
\end{align*}
Thus the variance of the arm difference shrinks linearly with the gap,
which explains why the correlation-aware hardness improves from the
classical $\Theta(1/\Delta^{2})$ scaling to $\Theta(1/\Delta)$ in this
maximally correlated regime. In this sense,
\cref{eqn:Bernstein-condition-analogue} plays the role of a
Bernstein condition for the score differences.

\paragraph{A Matching Lower Bound.}
To further motivate the hardness measure $H_3$, in \cref{sec:h3details} we also prove the following lower bound:
Consider the problem of identifying the better of two models with expected score gap $\Delta$ and variance of the score difference $V_{\star,2}$, using $n$ i.i.d.\ pairs of binary scores $\crl{(S_{1j_i},S_{2j_i})}_{i=1}^n$ from the two models. 
\begin{theorem}[Informal version of \cref{thm:lower-bound}]
\label{thm:lower-bound-informal}
    In the setting above, any algorithm returning $\Ihat$ incurs error at least 
\begin{equation*}
    e_n = \PP(\Ihat\neq \istar) \geq c_1 \exp\prn{-c_2\frac{n}{H_3}}
\end{equation*}
on at least one distribution of $(S_{1J},S_{2J})$ adhering to the gap and variance constraints.
Here $c_1,c_2>0$ are two universal constants.
\end{theorem}
The formal version of this statement can be found in \cref{thm:lower-bound}. In particular in the 2-model i.i.d.\ setting, if all we know is that the variance equals $V_{\star,2}$ 
and the gap equals $\Delta$, then $H_3$ measures the optimal sample complexity (in a minimax sense), which is achieved by SySRs up to constants. While a generalization to sampling without replacement is open and likely requires a finite sample correction similar to the Bernstein-Serfling inequality \cite{bardenet2015concentration}, we expect this correction to be minor when the benchmark dataset size $L$ is large compared to the budget $n$.

\section{Experiments}
\label{sec:experiments}
Having established theoretical guarantees for SySRs, we now validate its performance empirically, comparing it to algorithms from previous work.  
\subsection{Experimental Setup}
\paragraph{Datasets and Models.}
We base our experiments on publicly available per-query results for 15 standard benchmark datasets, sourced from 
HELM Lite \citep{liang2022holistic} and the OpenLLM Leaderboard \citep{beeching2023open}.
For each benchmark, we include all models for which evaluation results were available for all queries in December 2025. As all considered benchmarks use greedy decoding, all scores $S_{ij}$ are deterministic\footnote{Note that this does not make the rewards $R_i = S_{iJ}$ deterministic, as test set indices $J$ are still randomized.}. This is consistent with our theoretical model, as the special case in which the distributions $\phi_{ij}$ are point masses.

\cref{tab:dataset-diversity} summarizes key statistics across all datasets, showing substantial variation in both dataset size and difficulty in terms of the hardness measures $H_2$ and $H_3$. 
For more details on the datasets, see \cref{tab:dataset-details} in the appendix.

\paragraph{Baselines.}
We compare SySRs to the following baselines: 
\begin{itemize}
    \item \textbf{Uniform Sampling (US)} evaluates each model on equally many test instances, each independently drawn uniformly without replacement.
    \item \textbf{Synchronized Uniform Sampling (SyUS)} randomly selects test instances uniformly without replacement, evaluating all models on the same subset of instances. 
    \item \textbf{UCB-E} from \citet{audibert2010best}.
    \item \textbf{UCB-E with low-rank factorization (UCB-E-LRF)} from \citet{zhou2025speeding}    
    \item \textbf{Successive Rejects (SR)} from \citet{audibert2010best}.
    \item \textbf{Synchronized UCB-E (SyUCB-E)}: a novel correlation-aware variant of UCB-E, using a global task schedule similar to SySRs (see \cref{sec:smart-ucb-e}).
\end{itemize}
 We run all (non-LRF) UCB-E variants for the exploration parameters $a \in \{0.1, 1.0, 10.0\}$. 
For UCB-E-LRF, we evaluate two variants: one with the default hyperparameters from \citet{zhou2025speeding} (including a warm-up phase consisting of 5\% of the benchmark size $K\times L$), and one with no warm-up phase, which we call UCB-E-LRF(No Warm-up).

All algorithms use sampling \emph{without replacement}. We slightly modify all algorithms except for UCB-E-LRF to ensure well-defined behavior when a model has already been evaluated on the full set of queries $\mathcal{X}$. UCB-E-LRF already handles this case natively. For the other UCB-E variants, we simply exclude fully evaluated models from the arm selection. For SR and SySRs, we modify the schedule, capping $n_k$ at the number of queries in the test set $L$ and reallocating the freed-up budget proportionally to earlier phases. 
Additional implementation details for all algorithms can be found in \cref{sec:implementation-details}.

\begin{table}[h]
\centering
\small
\caption{Dataset diversity and hardness statistics across 15 benchmark datasets. The $H'_3/H_2$ ratio quantifies our tighter theoretical analysis, while the $H_3/H'_3$ ratio quantifies potential gains from exploiting correlation between models.
}
\label{tab:dataset-diversity}
\begin{tabular}{@{}lcccc@{}}
\toprule
\textbf{Metric} & \textbf{Min} & \textbf{Max} & \textbf{Mean} & \textbf{Std} \\
\midrule
Tasks $L$ & 437 & 12,032 & 2,298 & 3,145 \\
Models $K$ & 39 & 450 & 158 & 155 \\
$\Delta_2$ & 0.0020 & 0.0455 & 0.0143 & 0.0138 \\
$H_2$ & $1.7{\times}10^{3}$ & $5.0{\times}10^{5}$ & $8.7{\times}10^{4}$ & $1.3{\times}10^{5}$ \\
$H'_3$ & $1.0{\times}10^{3}$ & $8.7{\times}10^{4}$ & $2.5{\times}10^{4}$ & $2.4{\times}10^{4}$ \\
$H_3$ & $1.0{\times}10^{3}$ & $4.3{\times}10^{4}$ & $1.0{\times}10^{4}$ & $1.1{\times}10^{4}$ \\
$H'_3/H_2$ & 0.11 & 1.08 & 0.61 & 0.32 \\
$H_3/H'_3$ & 0.07 & 1.04 & 0.56 & 0.35 \\
\bottomrule
\end{tabular}
\end{table}

\paragraph{Evaluation Metrics.}
For each benchmark, we define the best-performing model $\istar$ as the one with the highest average performance across all test instances. 
As our primary evaluation metric, we measure the 
\emph{error probability} $e_n = \mathbb{P}(\Ihat \neq \istar)$, quantifying the likelihood that an algorithm fails to identify the best-performing model for a given evaluation budget $n$. We call $ 1 - e_n$ the
\emph{identification accuracy}. 
To obtain robust empirical estimates $\hat{e}_n$ for $e_n$, we perform $1000$ independent runs per algorithm. For UCB-E-LRF and UCB-E-LRF (No Warm-up), we reduce the number of runs to $100$ due to their high computational overhead. Each run uses a different random seed to sample the subset of queries used by the algorithms. 
We then estimate $\hat{e}_n$ as the fraction of runs in which the algorithm fails to identify the best model. 

To enable comparison across datasets of different sizes, we express the budget as a percentage $b = \frac{n}{K_D \times L_D}$ of the full evaluation budget, where $K_D$ and $L_D$ denote the number of models and test instances of dataset $D$, respectively. We then estimate the error probability at regular intervals of budget percentage $b$, denoting our estimates $\hat{\epsilon}_b$. 
Based on $\hat{\epsilon}_b$, we analyze the performance of algorithms in two ways 
that are comparable across datasets:
First, as in \cref{fig:mean}, we study how $\hat{\epsilon}_b$ changes with the budget percentage $b$. This provides us with insights into how well algorithms perform given a \emph{fixed percentage of the evaluation budget}.
Second, we define the \emph{X\% Confidence Budget} (CB) 
as the minimum budget percentage $b^\star$ such that the empirical identification accuracy $1-\hat{\epsilon}_b$ satisfies $1-\hat{\epsilon}_b \geq  X\%$ for all evaluated budget percentages $b \geq b^\star$. In other words, identification accuracy stays above $X\%$ once the budget percentage exceeds the X\% CB.
This dual view informs us how much evaluation budget an algorithm needs to achieve \emph{a target error rate}.
\subsection{Main Results}

We now present the main experimental results, focusing on both the identification accuracy view and the confidence budget view across datasets. Both views confirm SySRs' efficiency for model selection, as well as broader benefits of synchronization.

\begin{table}[h]
\centering
\small
\setlength{\tabcolsep}{4pt}
\caption{Identification accuracy averaged across datasets at different percentages of the model/query pairs in each dataset. Higher values indicate better performance. Best values in \textbf{bold}. }
\label{tab:identification-accuracy}
\begin{tabular}{lrrrrrrr}
\toprule
& \multicolumn{7}{c}{Budget Percentage} \\
\cmidrule(lr){2-8}
Algorithm & 8\% & 12\% & 16\% & 20\% & 25\% & 30\% & 35\% \\
\midrule
\multicolumn{8}{l}{\emph{Baselines}} \\
US & 30.2 & 36.7 & 42.0 & 45.3 & 50.5 & 54.2 & 58.0 \\
SyUS & 39.1 & 46.8 & 52.0 & 56.5 & 61.0 & 65.3 & 68.7 \\
\midrule
\multicolumn{8}{l}{\emph{Successive Rejects}} \\
SR & 72.4 & 89.7 & 95.9 & 98.2 & 99.3 & 99.8 & 99.9 \\
SySRs & \textbf{83.5} & \textbf{95.1} & \textbf{98.5} & \textbf{99.6} & \textbf{99.9} & \textbf{99.9} & \textbf{100.0} \\
\midrule
\multicolumn{8}{l}{\emph{UCB-E}} \\
$a=0.1$ & 61.0 & 66.9 & 70.7 & 73.0 & 75.8 & 77.7 & 80.1 \\
$a=1.0$ & 80.9 & 86.8 & 89.9 & 95.3 & 97.5 & 98.0 & 98.3 \\
$a=10.0$ & 51.9 & 63.3 & 75.5 & 85.2 & 91.1 & 91.8 & 93.0 \\
\midrule
\multicolumn{8}{l}{\emph{SyUCB-E}} \\
$a=0.1$ & 67.1 & 74.0 & 78.5 & 81.3 & 84.2 & 86.4 & 88.4 \\
$a=1.0$ & 83.4 & 89.4 & 92.8 & 96.1 & 98.0 & 98.7 & 98.9 \\
$a=10.0$ & 61.4 & 73.3 & 82.8 & 89.3 & 92.1 & 93.5 & 95.1 \\
\midrule
\multicolumn{8}{l}{\emph{UCB-E-LRF}} \\
5\% Warm-up & 47.2 & 60.8 & 73.2 & 81.0 & 85.9 & 91.7 & 93.2 \\
No Warm-up & 39.8 & 46.7 & 65.7 & 73.3 & 80.6 & 82.7 & 84.5 \\
\bottomrule
\multicolumn{8}{l}{\footnotesize All values in \%. Rounded down to nearest 0.1\%. } \\
\end{tabular}
\end{table}

\paragraph{The Average Identification Accuracy View.}
\cref{tab:identification-accuracy} shows the identification accuracy averaged across all 15 datasets at different budget percentages. Our hyperparameter-free SySRs 
outperforms all baselines for all tested hyperparameters at all considered budgets. In particular, at 12\% of the budget, SySRs achieves an average identification accuracy of 95\% while all other bandit methods stay below 90\% and the uniform sampling baselines do not even reach 50\%.

The results in \cref{tab:identification-accuracy} also demonstrate the impact of dataset synchronization, more broadly: For all of US, UCB-E, and SR, the synchronized variants outperform the non-synchronized version, often by around ten percentage points. In particular SyUCB-E, our synchronized version of UCB-E, gets close to SySRs' performance at some budgets, but remains sensitive to hyperparameter choice. 
For more fine-grained results on identification error, including per-benchmark results, see Appendix \ref{sec:all-dataset-plots}.

\begin{table}[h!]
\centering
\small
\setlength{\tabcolsep}{5pt}
\caption{Confidence Budget statistics across 15 datasets. For each confidence level, 
we report the average CB $b^{\star}$ and the maximum (worst-case) CB $b^{\star}$. Lower values indicate better performance. Best values in \textbf{bold}. Best hyperparameter setting \underline{underlined}.}
\label{tab:perfect-identification}
\begin{tabular}{lrr@{\hspace{12pt}}rr@{\hspace{12pt}}rr}
\toprule
& \multicolumn{2}{c@{\hspace{12pt}}}{90\% CB} & \multicolumn{2}{c@{\hspace{12pt}}}{95\% CB} & \multicolumn{2}{c}{100\% CB} \\
\cmidrule(lr){2-3} \cmidrule(lr){4-5} \cmidrule(lr){6-7}
Algorithm & Mean & Max & Mean & Max & Mean & Max \\
\midrule
\multicolumn{7}{l}{\emph{Baselines}} \\
US & 73 & 99 & 78 & 100 & 90 & 100 \\
SyUS & 59 & 98 & 66 & 100 & 83 & 100 \\
\midrule
\multicolumn{7}{l}{\emph{Successive Rejects}} \\
SR & 11 & 25 & 13 & 28 & 20 & 43 \\
SySRs & \textbf{8} & \textbf{17} & \textbf{10} & \textbf{19} & \textbf{17} & \textbf{35} \\
\midrule
\multicolumn{7}{l}{\emph{UCB-E}} \\
$a=0.1$ & 50 & 77 & 64 & 87 & 89 & 100 \\
$a=1.0$ & \underline{10} & \underline{23} & \underline{17} & \underline{40} & 68 & 99 \\
$a=10.0$ & 19 & 46 & 20 & 47 & \underline{22} & \underline{49} \\
\midrule
\multicolumn{7}{l}{\emph{SyUCB-E}} \\
$a=0.1$ & 37 & 71 & 50 & 82 & 89 & 100 \\
$a=1.0$ & \underline{9} & \underline{23} & \underline{14} & \underline{29} & 66 & 99 \\
$a=10.0$ & 17 & 43 & 18 & 46 & \underline{23} & \underline{52} \\
\midrule
\multicolumn{7}{l}{\emph{UCB-E-LRF}} \\
5\% Warm-up & \underline{18} & \underline{50} & \underline{19} & \underline{50} & \underline{24} & \underline{55} \\
No Warm-up & 19 & \underline{50} & 21 & \underline{50} & 26 & 60 \\
\bottomrule
\multicolumn{7}{l}{\footnotesize All values in \%.}\\
\multicolumn{7}{l}{\footnotesize LRF: rounded down to nearest checkpoint (\cref{sec:implementation-details}).}
\end{tabular}
\end{table}

\paragraph{The Average and Worst-Case Confidence Budget View.} 
To complement the budget-averaged perspective, we analyze algorithm performance using the Confidence Budget (CB) defined in the previous section. \cref{tab:perfect-identification} shows the mean and maximum of the CB across datasets for each algorithm at the 90\%, 95\%, and 100\% identification accuracy levels. SySRs consistently achieves the lowest mean and maximum CB values across all three accuracy thresholds, indicating its superior efficiency in identifying the best model with high confidence. In particular, across all datasets, SySRs needs to evaluate at most 35\% of model/query pairs to identify the best model with probability one. Meanwhile, SR requires 43\%, the best UCB variants need around 50\%, and the non-bandit baselines essentially require the full evaluation budget to achieve the same. 

As before, we observe that synchronized algorithms (SySRs, SyUCB-E and SyUS) outperform their non-synchronized counterparts (SR, UCB-E and US respectively) in almost all cases. This highlights the consistent benefits of leveraging model similarity through synchronized task selection. 

Beyond this, \cref{tab:perfect-identification} starkly illustrates the dependence of UCB-E variants on hyperparameter selection. While SR and SySRs perform well without any hyperparameters, the performance of UCB-E and SyUCB-E is strongly affected by the exploration parameter $a$. Although lower values of $a$ (e.g., $a=1.0$) tend to perform much better in terms of the 90\% CB, higher values (e.g., $a=10.0$) are needed to identify the best model with perfect reliability (i.e., to achieve strong performance in terms of 100\% CB). Again, more fine-grained results can be found in 
Appendix \ref{sec:all-dataset-plots}.

\paragraph{Additional Experiments.} Going beyond best model identification, \cref{sec:top-m-ranking} provides preliminary experiments on using SySRs and other bandit algorithms for top-$m$ model identification and  other model ranking tasks. We find that bandit algorithms continue to perform much better than uniform sampling. SySRs performs particularly well in terms of worst-case performance, and for small $m$, but SyUCB-E has an edge in terms of average error rates, especially when $m$ gets larger. 
In addition, \cref{sec:subset-selection} compares our approach to subset selection methods \citep{vivek2024anchor,polo2024tinybenchmarks}, which aim to construct smaller but predictive test sets (see \cref{sec:related-work}). As suggested by prior work \citep{zhang2025how,yauney2026}, these methods mostly bring minor improvements over uniform sampling, and evaluation on the resulting test sets yields much higher identification error rates than using SySRs. 
 
\section{Limitations and Future Work}
Our first guarantees on gains from model similarity in LLM selection open up multiple avenues for future research. First, we observe that despite the strong average performance of SySRs across datasets, UCB algorithms 
can outperform SySRs for some combinations of budgets and hyperparameters. 
This is akin to how classical UCB-E is often competitive with basic SR for appropriately chosen hyperparameters (Audibert et al., 2010) and motivates additional
research into how model similarity can be effectively used by UCB algorithms. Our SyUCB-E baseline is a first step in this direction, as it consistently outperforms standard UCB-E, but the proof techniques from \cref{thm:csr_bound} are not directly applicable to SyUCB-E.

Second, despite its relatively weak average performance, we observe UCB-E-LRF \citep{zhou2025speeding} performing very well on a small subset of benchmarks, even with minimal data. This motivates future work on designing more robust algorithms based on low-rank factorization or other imputation methods. Ideally, these algorithms could realize the occasional strong gains from LRF, while failing more gracefully in other cases and coming with theoretical guarantees. 

Third, our work focuses on best model identification, but more general ranking objectives can be relevant in practice. A natural generalization is identifying or ranking the top $m$ models. We provide preliminary experiments in \cref{sec:top-m-ranking}, observing that SySRs consistently outperforms SR. This indicates that synchronized evaluation might be useful for ranking models, more broadly. However, neither  algorithm is explicitly designed for top-$m$ ranking. As such, we see synchronized versions of existing top-$m$ algorithms such as Successive Accepts and Rejects (SAR) \citep{pmlr-v28-bubeck13} as a promising direction for future work.  

\section{Conclusion}
In this work, we propose SySRs, a hyperparameter-free best-arm identification algorithm that exploits model similarity. We provide theoretical guarantees for SySRs that get stronger as the similarity among evaluated models increases. SySRs demonstrates strong empirical performance in both an average and a worst-case sense. Across a suite of 15 standard LLM benchmarks, SySRs reliably identifies the best model by evaluating at most $35\%$ of model/query pairs, greatly reducing the number of queries required for model selection. Beyond this, we believe that the structure of SySRs makes it well-suited for LLM inference: While UCB algorithms often choose to evaluate a different model at every step, SySRs evaluates each remaining model on a larger number of queries in each phase. This enables straightforward batching and can reduce overhead from reloading model weights when switching models. 

\section*{Acknowledgements}
Florian Dorner is grateful for financial support from the Max Planck ETH Center for Learning Systems (CLS). Zifan Lyu and Tobias Wegel were supported by SNSF Grant 204439. 

\ifthenelse{\boolean{icml}}{
\section*{Impact Statement}
This paper presents work whose goal is to advance the field of machine learning. There are many potential societal consequences of our work, none of which we feel must be specifically highlighted here.
}{}
 
\bibliography{example_paper}
\bibliographystyle{icml2026}
\onecolumn

\appendix
\ifthenelse{\boolean{incappendix}}
{\tableofcontents
\addtocontents{toc}{\protect\setcounter{tocdepth}{2}}

\begin{table}[t]
\centering
\caption{Notation. }
\label{tab:example}
\begin{tabular}{cl}
\toprule
Symbol & Meaning / Definition \\
\midrule
$\mathcal{F}=\crl{f_1,\ldots,f_K}$ & Models / arms \\
$\cX=\crl{x_1,\ldots,x_L}$ & Tasks \\
$n$ & Total budget \\
$i$ & Indexing rows / models / arms \\
$(i)$ & Performance-sorted indexing rows / models / arms \\
$j$ & Indexing columns / tasks \\
$S_{ij}$ & Score of model $i$ on task $j$ (with values in $[0,1]$) \\
$\phi_{ij}$ & Distribution of $S_{ij}$ \\
$\S=(S_{ij})\in [0,1]^{K\times L}$ & Score matrix  \\
$t\in[n]$ & Time index \\
$i_t$ & Arm pulled by algorithm at time $t$ \\
$j_t$ & index of test instance at time $t$  \\ 
$\mu_i$ & Mean of $i$-th row \\
$\istar$ & Index of best arm \\

$\mu^\star$ & $\mu^\star = \mu_{\istar} = \max_{i} \mu_i$\\
$R_i$ & Random variable $\sim \Uniform(\S_{iJ})$\\
$\sr=(R_1,\ldots,R_K)$ & Joint reward vector \\
$\IhatSc$ & Output of Synchronized Successive Rejects \\
$\Var$ & Variance \\
$V_i$ & Variance of score $\Var_{J\sim \text{Unif}(1,\cdots, L)}(S_{(i) J})$\\
$V_{\star,i}$ & Variance of score difference $\Var_{J\sim \text{Unif}(1,\cdots, L)}(S_{\istar J} - S_{(i)J})$\\
$\text{Cov}_{\star,i}$ & Covariance of scores $\mathrm{Cov}_{J\sim \text{Unif}(1,\cdots, L)} (S_{\istar J}, S_{(i)J})$\\
$\Delta_i$ & Sub-optimality gap, $\Delta_i = \mu^\star - \mu_{(i)}$ (with $\Delta_1 = \Delta_2$) \\
$e_n$ &  Theoretical best-arm identification error at budget $n$\\ 
$\hat{e}_n$ & Estimated best-arm identification error at budget $n$  \\ 
$\hat{\epsilon}_b$ & Estimated best-arm identification error at budget percentage $b = \frac{n}{K\times L}$  \\ 
\bottomrule
\end{tabular}
\end{table}

\counterwithin{figure}{section}
\counterwithin{table}{section}
\counterwithin{algorithm}{section}

\crefalias{section}{appendix}
\crefalias{subsection}{appendix}
\crefalias{subsubsection}{appendix}
\crefname{appendix}{appendix}{appendices}
\Crefname{appendix}{Appendix}{Appendices}

\section{Detailed Analysis of the New Hardness Measure $H_3$}\label{sec:h3details}
In this section, we characterize the \emph{regime of improvement}, specifically determining the minimum level of correlation required for our correlation-aware complexity $H_3$ to be strictly smaller than the standard $H_2$, and we state the formal version of the lower bound in \cref{thm:lower-bound-informal}.

\paragraph{Regime of Improvement.} Recall that both hardness measures are defined as maxima over the suboptimal arms:
\[
H_3 \;=\; \max_{i\in\{2,\dots,K\}} i  H_{3,i}
\quad \text{and} \quad
H_2 \;=\; \max_{i\in\{2,\dots,K\}} i H_{2,i}
\]
where we set $ H_{2,i} \coloneqq \Delta_i^{-2} $. While $H_2$ is clearly unaffected by any correlation between the arms, expanding the variance in \cref{eq:H3k} yields
\[
\begin{aligned}
H_{3,i}
&=
\frac{2 V_1 + 2 V_i +\frac{2}{3}(1+\Delta_i)\Delta_i -4\,\text{Cov}_{\star,i}}{\Delta_i^2}
\end{aligned}
\]
where $V_i =  \Var_{J\sim \text{Unif}([L])}(S_{(i) J})$ and $\Cov_{\star, i}=\Cov_{J\sim \text{Unif}([L])}(S_{i^\star J},S_{(i) J})$. Thus, $H_{3,i}$
strictly decreases with covariance for fixed marginal score distributions. 
We again express the (positive) covariance between the optimal arm $\istar$ and a suboptimal arm $i$ as a fraction of the maximal possible covariance for the mean arm scores:
\[
 \theta_i\;=\;   \frac{\Cov_{\star,i}}{\Cov_{\max}(\mu^{\star}, \Delta_i) }
\]
where we used the shorthand notation $\mathrm{Cov}_{\max}(\mu^{\star}, \Delta_i)  = (\mu^{\star}-\Delta_i)(1-\mu^{\star})$ for the maximal achievable covariance between binary score distributions with means $\mu^\star$ and $\mu^\star-\Delta_i$, and we assume that $\theta_i\in [0,1]$. Since  $H_{3,i}$ decreases in $\theta_i$ for fixed $V_1$ and $V_i$, there exists a critical value $\theta_i^\star$ such that $H_{3,i} < H_{2,i}$ if and only if $\theta_i > \theta_i^\star$.
The following proposition makes this statement rigorous.

\begin{proposition}[Correlation Threshold for Improvement]
\label{thm:theta-star}
For any arm $i \neq 1$, whenever the correlation coefficient $\theta_i$ exceeds the threshold \[
\theta^\star(\mu^{\star},\Delta_i) \;:=\;
\frac{
-\frac{1}{3}\Delta_i^2 + \Delta_i\mu^{\star} - \frac{1}{3}\Delta_i - (\mu ^\star)^2 + \mu^{\star} - \frac{1}{4}
}{
(\mu^{\star}-\Delta_i)(1-\mu^{\star})
},
\]
the correlation-aware hardness $H_{3,i}$ is strictly smaller than the standard hardness $H_{2,i}$.
\end{proposition}
The proof is detailed in \Cref{sec:proofs}. As a consequence of Proposition \ref{thm:theta-star}, if every suboptimal arm satisfies the threshold condition, the correlation-aware hardness $H_{3}$ is smaller than the standard hardness $H_{2}$. Moreover, we can derive \cref{cor:sufficient_conditions}.

The inequality $H_3 < H_2$ implies a strictly tighter upper bound on the probability of error. Since the error rate scales as $\exp(-n/H)$, even a constant factor reduction in the hardness parameter $H$ results in an exponential decay of the error probability for a fixed budget $n$. Conversely, this means the algorithm requires significantly fewer samples to 
identify the best arm with high probability, when the correlation between models is sufficiently high (i.e. $\theta_i > \theta^\star$).

\paragraph{A Lower Bound.}
To further justify the hardness measure $H_3$, we now show the formal version of \cref{thm:lower-bound-informal}. In the following result, we demonstrate that in a two-arm problem with correlated arms, in order to identify the better arm from $n$ i.i.d.\ paired score samples, any algorithm has to make an error that scales exactly as suggested by \cref{thm:csr_bound}.
Specifically, denote the function $H_3(\Delta,V) = (2V +\tfrac{2}{3}(1+\Delta)\Delta)/\Delta^2$. Then we can show the following.
\begin{theorem}\label{thm:lower-bound}
    For any $\Delta\in(0,1/2)$ and $s\in[\Delta,1/2]$, there exist distributions $P$ and $Q$ on $\crl{0,1}^2$ so that:
    First, $\EE_{(X,Y)\sim P}\brk{X-Y} = \Delta$, that is, under $P$, the best model is $\istar = 1$ and the gap to the second-best model is $\Delta$, and
        $\EE_{(X,Y)\sim Q}\brk{X-Y} = -\Delta$, that is, under $Q$, the best model is $\istar = 2$ and the gap to the second-best model is $\Delta$.
    Second, under both $P$ and $Q$, the variance is $V=s-\Delta^2$, that is $\Var_{(X,Y)\sim P}\brk{X-Y}=\Var_{(X,Y)\sim Q}\brk{X-Y}=V$. 
    And third, for any $n\in\NN$ and any algorithm $\Ihat:(\crl{0,1}^2)^n\to \crl{1,2}$ that, based on $n$ i.i.d.\ samples $\cD$ from $P$ or $Q$, aims to identify the better model as $\Ihat(\cD)\in \crl{1,2}$, incurs identification error at least $\frac{1}{4}\exp\prn{-11\frac{n}{H_3(\Delta,V)}}$ on at least one of $P$ or $Q$. That is, for all $n\in \NN$,
    \begin{equation*}
        \inf_{\Ihat:(\crl{0,1}^2)^n\to \crl{1,2}}\max\crl{\PP_{\cD\sim P^n}\prn{\Ihat(\cD)\neq 1},\PP_{\cD\sim Q^n}\prn{\Ihat(\cD)\neq 2}} \geq \frac{1}{4}\exp\prn{-11\frac{n}{H_3(\Delta,V)}}.
    \end{equation*}
\end{theorem}
Note that the gap $\Delta$ and variance $V$ here do \emph{not} depend on the sample size $n\in\NN$. The proof is in \Cref{proof:lower-bound}.

The key aspect of \cref{thm:lower-bound} is that the difficulty of the problem is determined not only by the gap $\Delta$, but also by the variance $V = \Var\brk{X-Y}$, which captures the correlation structure between the two arms. Unlike classical bounds that depend only on marginal variances, here the hardness is expressed in terms of the variability of the difference $X-Y$. This means that if the two arms are strongly positively correlated, then $X-Y$ has low variance and the problem becomes easier, whereas weak or negative correlation increases $V$ and thus the difficulty. The theorem shows that this is unavoidable and any algorithm using IID samples must incur an error that scales with $H_3(\Delta,V)$. Exploiting correlation is not just beneficial but information-theoretically necessary for optimal performance, and matches \cref{thm:csr_bound}. 

Extending this bound to multiple arms is much more challenging, as while we can simulate $n$ arm pulls from $n$ paired evaluations of all $K$ arms, this comes at an overhead of $K$ model evaluations per arm pull; we leave the $K$-arm case as interesting future work in the line of works such as \citet{kaufmann2016complexity,carpentier2016tight}.

\section{Deferred Proofs}
\label{sec:proofs}

In this section we provide the deferred proofs for \cref{thm:csr_bound,thm:lower-bound,thm:theta-star,cor:sufficient_conditions}.
To prove \cref{thm:csr_bound}, we first provide two key technical concentration bounds. 

\subsection{Two Concentration Bounds}

We use the following key lemma, a modified version of Bernstein's inequality \citep{boucheron2003concentration} that holds for sampling without replacement.
Note that the following bound is slightly sharper than Proposition 1.4 in \citet{bardenet2015concentration} in the denominator of the exponential.

For a finite population indexed by $\{1, \ldots, N\}$ and a sample size $1 \leq n \leq N$, we denote in the following section by
\begin{itemize}
  \item $\pi_w$: an \emph{injection} $\{1, \ldots, n\} \to \{1, \ldots, N\}$ drawn uniformly at random (sampling \emph{without} replacement);
  \item $\pi_r$: a function $\{1, \ldots, n\} \to \{1, \ldots, N\}$ whose values are i.i.d.\ uniform on $\{1, \ldots, N\}$ (sampling \emph{with} replacement).
\end{itemize} 

\begin{lemma}[Bernstein bound for bounded variables without replacement]
\label{lem:bernstein-bounded-wor}
Let $\mathcal{Z} = \{\phi_1, \ldots, \phi_N\}$ be a finite population of probability distributions such that for all $i$, $\phi_i$ takes values in $[-M, M]$, and such that these distributions verify
\[
  \frac{1}{N}\sum_{i=1}^N \mathbb{E}_{Z \sim \phi_i}[Z] \;=\; 0,
  \qquad
  \frac{1}{N}\sum_{i=1}^N \mathbb{E}_{Z \sim \phi_i}[Z^2] \;=\; \sigma^2.
\]
For any $1 \leq n \leq N$, let $(Z_t)_{t=1}^n$ be sampled by first drawing an injection $\pi_w$ uniformly at random, and then drawing $Z_t \sim \phi_{\pi_w(t)}$ independently for each $t$. We define the empirical mean $\widehat{\mu}_{Z,n} := \tfrac{1}{n}\sum_{t=1}^n Z_t$. For any $\varepsilon > 0$, it holds that
\begin{equation}
  \mathbb{P}\bigl(\widehat{\mu}_{Z,n} \geq \varepsilon\bigr)
  \;\leq\; \exp\!\left(-\frac{n\varepsilon^2}{2\sigma^2 + \tfrac{2}{3}M\varepsilon}\right). \label{eq:bernstein-stochastic}
\end{equation}
\end{lemma}

\begin{proof}[Proof of Lemma~\ref{lem:bernstein-bounded-wor}]The proof follows by combining a standard moment generating function bound, e.g., from \citet{boucheron2003concentration}, with Hoeffding's reduction; see \citet[Theorem~4]{hoeffding1963probability} or \citet[Lemma~1.1]{bardenet2015concentration}.

To develop the analysis, we begin by introducing a substitute $(\widetilde{Z}_t)_{t=1}^n$ of $(Z_t)_{t=1}^n$, sampled i.i.d.\ uniformly \emph{with} replacement from the uniform mixture distribution $\bar{\phi} = \tfrac{1}{N}\sum_{i=1}^N \phi_i$, i.e. $\widetilde{Z}_t \sim \phi_{\pi_r(t)}$ conditionally on $\pi_r$ . We have $\widetilde{Z}_t \sim \bar{\phi}$, yielding $\mathbb{E}[\widetilde{Z}_t] = 0$, $|\widetilde{Z}_t| \leq M$, and $\mathbb{E}[\widetilde{Z}_t^{\,2}] = \sigma^2$. Therefore, for every $k \geq 2$, it holds $|\widetilde{Z}_t|^k \leq M^{k-2}\,\widetilde{Z}_t^{\,2}$, which implies that
\begin{equation}
\label{eq:mom-bound-stoch}
\mathbb{E}\!\left[\,|\widetilde{Z}_t|^k\,\right]
\;\leq\;
\sigma^2 M^{k-2}
\;\leq\;
\frac{\sigma^2}{2}\,k!\left(\frac{M}{3}\right)^{\!k-2},
\qquad \forall k \geq 2.
\end{equation}
Now, fix $\lambda > 0$ small enough such that $u := \lambda M / 3 \in (0, 1)$. Using $\mathbb{E}[\widetilde{Z}_t] = 0$ and inequality~\eqref{eq:mom-bound-stoch}, we get that
\begin{align*}
\mathbb{E}[e^{\lambda \widetilde{Z}_t}]
&=
1 + \sum_{k \geq 2} \frac{\lambda^k}{k!}\,\mathbb{E}[\widetilde{Z}_t^{\,k}]
\;\leq\;
1 + \sum_{k \geq 2} \frac{\lambda^k}{k!}\,\mathbb{E}[|\widetilde{Z}_t|^k]
\;\leq\;
1 + \sum_{k \geq 2} \lambda^k \frac{\sigma^2}{2}\left(\frac{M}{3}\right)^{\!k-2} \\
&=
1 + \frac{\lambda^2 \sigma^2}{2}\sum_{k \geq 2}\!\left(\frac{\lambda M}{3}\right)^{\!k-2}
\;=\;
1 + \frac{\lambda^2 \sigma^2}{2}\cdot \frac{1}{1-u}
\;\leq\;
\exp\!\left(\frac{\lambda^2 \sigma^2}{2(1-u)}\right).
\end{align*}
And by independence, we can use this bound to obtain
\begin{equation}
\label{eq:mgf-wr-stoch}
\mathbb{E}\!\left[\exp\!\Bigl(\lambda \sum_{t=1}^n \widetilde{Z}_t\Bigr)\right]
\;\leq\;
\exp\!\left(\frac{n\sigma^2 \lambda^2}{2(1-u)}\right).
\end{equation}

We now come back to the sample without replacement employing Hoeffding's reduction~\citep{hoeffding1963probability}. Defining $\varphi_i(\lambda) := \log \mathbb{E}_{Z \sim \phi_i}\!\left[e^{\lambda Z}\right]$, conditional independence of $(Z_t)_{t=1}^n$ given $\pi_w$ yields
\begin{align*}
  \mathbb{E}\!\left[\exp\!\Bigl(\lambda \sum_{t=1}^n Z_t\Bigr)\right]
  &\;=\;
  \mathbb{E}_{\pi_w}\!\!\left[\mathbb{E}\!\left[\exp\!\Bigl(\lambda \sum_{t=1}^n Z_t\Bigr)\,\Big|\, \pi_w\right]\right]
  \;=\;
  \mathbb{E}_{\pi_w}\!\!\left[\prod_{t=1}^n \mathbb{E}_{Z \sim \phi_{\pi_w(t)}}\!\left[e^{\lambda Z}\right]\right] \\
  &\;=\;
  \mathbb{E}_{\pi_w}\!\!\left[\exp\!\Bigl(\sum_{t=1}^n \varphi_{\pi_w(t)}(\lambda)\Bigr)\right],
\end{align*}
and likewise $\mathbb{E}\!\left[\exp\!\bigl(\lambda \sum_t \widetilde{Z}_t\bigr)\right] = \mathbb{E}_{\pi_r}\!\!\left[\exp\!\bigl(\sum_t \varphi_{\pi_r(t)}(\lambda)\bigr)\right]$. Since $x \mapsto e^x$ is convex, sampling without replacement is dominated by sampling with replacement (Lemma~1.1 in \citet{bardenet2015concentration}) applied to the finite population $\{\varphi_i(\lambda)\}_{i=1}^N$, meaning that
\[
  \mathbb{E}\!\left[\exp\!\Bigl(\lambda \sum_{t=1}^n Z_t\Bigr)\right]
  \;=\;
  \mathbb{E}_{\pi_w}\!\!\left[\exp\!\Bigl(\sum_{t=1}^n \varphi_{\pi_w(t)}(\lambda)\Bigr)\right]
  \;\leq\;
  \mathbb{E}_{\pi_r}\!\!\left[\exp\!\Bigl(\sum_{t=1}^n \varphi_{\pi_r(t)}(\lambda)\Bigr)\right]
  \;=\;
  \mathbb{E}\!\left[\exp\!\Bigl(\lambda \sum_{t=1}^n \widetilde{Z}_t\Bigr)\right]
  \;\leq\;
  \exp\!\left(\frac{n\sigma^2 \lambda^2}{2(1-u)}\right),
\]
where the last inequality is~\eqref{eq:mgf-wr-stoch}.

Hence for any $s \geq 0$, using Markov's inequality with $S_n = \sum_{t=1}^n Z_t$ yields
\begin{align*}
\mathbb{P}\!\left(
S_n \geq \frac{n\sigma^2 \lambda}{2(1-u)} + \frac{s^2}{\lambda}
\right)
&=
\mathbb{P}\!\left(e^{\lambda S_n} \geq \exp\!\Bigl(\frac{n\sigma^2 \lambda^2}{2(1-u)} + s^2\Bigr)\right)
\;\leq\;
e^{-s^2}.
\end{align*}
To conclude, fix $\varepsilon > 0$ and choose $u := \frac{M\varepsilon}{3\sigma^2 + M\varepsilon}$, $\lambda := \frac{\varepsilon(1-u)}{\sigma^2}$, and $s := \frac{n\varepsilon\sqrt{1-u}}{\sqrt{2n\sigma^2}}$.
Then $u = \lambda M / 3 \in (0, 1)$ and $\frac{n\sigma^2 \lambda}{2(1-u)} + \frac{s^2}{\lambda} = n\varepsilon$, as well as $s^2 = \frac{(n\varepsilon)^2}{2n\sigma^2 + \tfrac{2}{3}Mn\varepsilon}$.
Therefore, we obtain that
\[
\mathbb{P}(S_n \geq n\varepsilon)
\;\leq\;
\exp\!\left(-\frac{n\varepsilon^2}{2\sigma^2 + \tfrac{2}{3}M\varepsilon}\right).
\] which concludes the proof.
\end{proof}

Next, we can apply Lemma \ref{lem:bernstein-bounded-wor} to derive the following bound for paired differences.

\begin{lemma}[Bernstein bound for paired differences]
\label{lem:bernstein-paired}
Let $\mathcal{D} = \crl{(\phi_{X 1}, \phi_{Y 1}), \ldots, (\phi_{X N}, \phi_{Y N})}$ be a finite population of paired distributions, where all distributions have support in $[0, 1]$, and 
\[
  \mu_X := \frac{1}{N}\sum_{i=1}^N \mathbb{E}_{X \sim \phi_{X i}}[X],
  \qquad
  \mu_Y := \frac{1}{N}\sum_{i=1}^N \mathbb{E}_{Y \sim \phi_{Y i}}[Y].
\]
For any $1 \leq n \leq N$, let $(X_t, Y_t)_{t=1}^n$ be sampled by first drawing an injection $\pi_w$ uniformly at random, and then drawing $(X_t, Y_t) \sim (\phi_{X  \pi_w(t)}, \phi_{Y  \pi_w(t)})$ independently for each $t$, and let the empirical means be
\[
  \widehat{\mu}_{X,n} := \frac{1}{n}\sum_{t=1}^n X_t,
  \qquad
  \widehat{\mu}_{Y,n} := \frac{1}{n}\sum_{t=1}^n Y_t.
\]
Then, if $\Delta := \mu_X - \mu_Y \geq 0$ and $ \sigma^2 = \frac{1}{N}\sum_{i=1}^N \mathbb{E}_{(X,Y) \sim (\phi_{X i}, \phi_{Y i})}\!\left[\bigl((Y - X) + \Delta\bigr)^{\!2}\right]$ 
it holds that
\begin{equation}
\label{eq:bernstein-pairs}
\mathbb{P}\left(\widehat{\mu}_{Y,n} \geq \widehat{\mu}_{X,n}\right)
\;\leq\;
\exp\left(
-\frac{n\Delta^2}{2\sigma^2 + \tfrac{2}{3}(1 + \Delta)\Delta}
\right)
.
\end{equation}
\end{lemma}

\begin{proof}[Proof of Lemma~\ref{lem:bernstein-paired}]
We define the paired differences $Z_t := Y_t - X_t + \Delta$ and the distributions $\phi_i$ as the law of $Y - X + \Delta$ when $(X, Y) \sim (\phi_{X i}, \phi_{Y i})$. Their population mean is
\[
\frac{1}{N}\sum_{i=1}^N \mathbb{E}_{Z \sim \phi_i}[Z]
=
\mu_Y - \mu_X + \Delta
=
0,
\]
and, since $X, Y \in [0, 1]$ and $\Delta \in [0, 1]$, the $Z_t$ are bounded by

\[
|Z_t| \leq 1 + \Delta =: M.
\]
Moreover,
\[
\frac{1}{N}\sum_{i=1}^N \mathbb{E}_{Z \sim \phi_i}[Z^2]
=
\sigma^2.
\]
Finally,
\[
\widehat{\mu}_{Y,n} \geq \widehat{\mu}_{X,n}
\iff
\sum_{t=1}^n (Y_t - X_t) \geq 0
\iff
\sum_{t=1}^n Z_t \geq n\Delta
\iff
\widehat{\mu}_{Z,n} := \frac{1}{n}\sum_{t=1}^n Z_t \geq \Delta.
\]
We can therefore apply Lemma~\ref{lem:bernstein-bounded-wor} to $(Z_t)_{t=1}^n$
with parameters $M = 1 + \Delta$, variance proxy $\sigma^2$, and $\varepsilon = \Delta$.
This yields
\[
\mathbb{P}\!\left(\widehat{\mu}_{Y,n} \geq \widehat{\mu}_{X,n}\right)
=
\mathbb{P}\!\left(\widehat{\mu}_{Z,n} \geq \Delta\right)
\leq
\exp\!\left(
-\frac{n\Delta^2}{2\sigma^2 + \tfrac{2}{3}(1 + \Delta)\Delta}
\right),
\]
which is exactly~\eqref{eq:bernstein-pairs}.
\end{proof}

\subsection{Proof of \cref{thm:csr_bound}}
As in the standard analysis of Successive Rejects, we observe that at the start of phase $k$, exactly $k-1$ arms have been removed, and exactly $K - (k-1) = K - k + 1$ arms are still active. Therefore, at least one of the $k$ worst arms must still be in the active set.
Therefore, if the optimal arm $i^\star$ is dismissed at the end of phase $k$, it must be that its empirical mean was smaller than the empirical mean of one of these surviving suboptimal arms.
Formally, if at the end of phase $k$ the optimal arm is removed, there exists some $i \in \{K+1-k, \ldots, K\}$ such that:
\[
\widehat{\mu}_{i^\star,n_k} \;\leq\; \widehat{\mu}_{(i),n_k}.
\]
By a union bound over all phases and all potential suboptimal arms in those phases, the error probability $e_n$ is bounded by:
\begin{equation}
\label{eq:en-union}
e_n
\;\leq\;
\sum_{k=1}^{K-1}\ \sum_{i=K+1-k}^{K}
\mathbb{P}\!\left(\widehat{\mu}_{(i),n_k} \geq \widehat{\mu}_{i^\star,n_k}\right).
\end{equation}
Fix a phase $k$ and an arm $i \in \{K+1-k, \ldots, K\}$. We observe stochastic rewards $X_t \sim \phi_{i^\star  \pi_w(t)}$ and $Y_t \sim \phi_{(i) \pi_w(t)}$, where $\pi_w$ is the uniform injection drawn.
This allows us to apply Lemma~\ref{lem:bernstein-paired} with $\Delta = \Delta_i \geq 0$ and $\sigma^2 = V_{\star, i}$. Combined with substituting the arm-specific hardness measure we get
\[
\mathbb{P}\!\left(\widehat{\mu}_{(i),n_k} \geq \widehat{\mu}_{i^\star,n_k}\right)
\;\leq\;
\exp\!\left(
-\frac{n_k \Delta_i^2}{2\, V_{\star,i} + \tfrac{2}{3}(1 + \Delta_i)\Delta_i}
\right)
\;=\;
\exp\!\left(-\frac{n_k}{H_{3,i}}\right).
\]
Plugging this back into~\eqref{eq:en-union} gives:
\[
e_n
\;\leq\;
\sum_{k=1}^{K-1}\ \sum_{i=K+1-k}^{K}
\exp\!\left(-\frac{n_k}{H_{3,i}}\right).
\]
From the definitions of $n_k = \Big\lceil \frac{1}{\overline{\log}(K)}\frac{n-K}{K+1-k}\Big\rceil$ and $H_3 = \max_{i \in \{2, \ldots, K\}} i\, H_{3,i}$, and using $i \geq K + 1 - k$, we have:
\[
\frac{n_k}{H_{3,i}}
\;\geq\;
\frac{n-K}{\overline{\log}(K)(K+1-k)} \cdot \frac{i}{H_3}
\;\geq\;
\frac{n-K}{\overline{\log}(K)\, H_3}.
\]
Substituting this uniform lower bound into the sum yields
\[
e_n
\;\leq\;
\sum_{k=1}^{K-1}\ \sum_{i=K+1-k}^{K} \exp\!\left(-\frac{n-K}{\overline{\log}(K)\, H_3}\right)
\;\leq\;
\frac{K(K-1)}{2}\exp\!\left(-\frac{n-K}{\overline{\log}(K)\, H_3}\right),
\]
which concludes the proof of~\cref{thm:csr_bound}.

\subsection{Proof of Proposition \ref{thm:theta-star}}

Fix a suboptimal arm \(i\) and  \(\Delta_i := \mu^\star - \mu_i > 0\).
We consider the two random variables \(S_{i^\star J}\) and \(S_{(i) J}\) supported in \([0,1]\) induced by the score matrix $\mathbf{S}$ and indices $J$ sampled uniformly without replacement from ${1,\cdots,L}$.
We assume the covariance is parametrized by \(\theta_i \in [0,1]\):
\[
\text{Cov}_{\star,i}
\;=\; \theta_i\,(\mu^\star-\Delta_i)(1-\mu^\star)
\;=\; \theta_i\,\mathrm{Cov}_{\max}(\mu^\star, \Delta_i).
\]

Recall the definitions of the arm-specific hardness components:
\[
H_{2,i} = \frac{1}{\Delta_i^2}, \qquad
H_{3,i} = \frac{2V_{\star,i} + \frac{2}{3}(1+\Delta_i)\Delta_i}{\Delta_i^2},
\]

Assume that 
\[
\theta_i > \theta^\star(\mu^\star,\Delta_i)
\]
or equivalently:
\[
4\left(  - (\mu^\star)^2 - \frac{1}{3}\Delta_i + \mu^\star + \mu^\star\Delta_i - \frac{1}{3}\Delta_i^2 - \frac{1}{4} \right)
\;<\; 4\theta_i(\mu^\star-\Delta_i)(1-\mu^\star)
\]
by rearranging the terms we get
\[
4\left(  - (\mu^\star)^2 - \frac{1}{3}\Delta_i + \mu^\star + \mu^\star\Delta_i - \frac{1}{3}\Delta_i^2  \right)
- 4\theta_i(\mu^\star-\Delta_i)(1-\mu^\star)
\;<\; 1
\]
and thus
\[
2\Bigl( \mu^\star(1-\mu^\star) + (\mu^\star-\Delta_i)(1-\mu^\star+\Delta_i)
- 2\theta_i(\mu^\star-\Delta_i)(1-\mu^\star) \Bigr)
+\frac{2}{3}(1+\Delta_i)\Delta_i
\;<\; 1.
\]

Now, since \(S_{i^\star J},S_{(i) J}\in[0,1]^2\) with means $\mustar$ and $\mustar-\Delta_i$ respectively, we have the variance bounds
\[
\mathrm{V}_1 \le \mu^\star(1-\mu^\star),
\qquad
\mathrm{V}_i \le (\mu^\star-\Delta_i)(1-\mu^\star+\Delta_i),
\]

Therefore the previous inequality implies
\begin{equation}
\label{eq:h3-h2-condition}
2V_{\star,i} + \frac{2}{3}(1+\Delta_i)\Delta_i \;<\; 1.
\end{equation}
Finally, \eqref{eq:h3-h2-condition} is equivalent to \(H_{3,i} < H_{2,i}\) by definition, which yields
\[
\theta_i > \theta^\star(\mu^\star,\Delta_i) \quad \implies \quad H_{3,i} < H_{2,i}.
\]
Since \(H_2 = \max_{k} k H_{2,k}\) and \(H_3 = \max_{k} k H_{3,i}\), we conclude:
\[
\left( \forall i \neq 1, \quad \theta_i > \theta^\star(\mu^\star, \Delta_i) \right) \implies H_3 < H_2.
\]

\subsection{Proof of \cref{thm:lower-bound}}
\label{proof:lower-bound}

We begin by defining the distributions $P,Q$. For that, we only look at the points $\cZ=\crl{(0,0),(0,1),(1,0)}\subset \crl{0,1}^2$. We define the distributions $P,Q$ as
\begin{align*}
    P(\crl{(1,0)})&=\frac{s+\Delta}{2}, \quad P(\crl{(0,1)})=\frac{s-\Delta}{2}, \quad \text{and} \quad  P(\crl{(0,0)})=1-s,  \\
    Q(\crl{(1,0)})&=\frac{s-\Delta}{2}, \quad Q(\crl{(0,1)})=\frac{s+\Delta}{2}, \quad \text{and} \quad  Q(\crl{(0,0)})=1-s.
\end{align*}
A simple computation of the expectations and variances then shows the first two parts of the theorem.

To prove the lower bound, we use the fact that by the Neyman-Pearson lemma, the binary testing problem between two distributions has minimal average error, measured in total variation distance $\TV$,
\begin{equation*}
    \inf_{\Ihat:(\crl{0,1}^2)^n\to\crl{1,2}} \crl{\frac{1}{2}\PP_{\cD\sim P^n}\prn{\Ihat(\cD)\neq 1}+\frac{1}{2}\PP_{\cD\sim Q^n}\prn{\Ihat(\cD)\neq 2}} = \frac{1-\TV(P^n,Q^n)}{2}.
\end{equation*}
Hence, it suffices to upper bound the total variation distance. Define $\BC(P,Q)=\sum_{z\in\cZ} \sqrt{P(z)Q(z)}$ and $H^2(P,Q)= \sum_{z\in\cZ} (\sqrt{P(z)}-\sqrt{Q(z)})^2$. One can verify that
\begin{equation*}
    H^2(P,Q) = 2\bigl(1-\BC(P,Q)\bigr),
\end{equation*}
so for product distributions $H^2(P^n,Q^n) = 2(1-\BC(P,Q)^n)$.
Lemma~2.3 in \citep{tsybakov2003introduction} gives $\frac{H^2}{2} \leq \TV \leq H\sqrt{1-\frac{H^2}{4}}$.
Applying the upper bound to $P^n,Q^n$ and substituting the product formula yields
\begin{equation*}
    \TV(P^n,Q^n)
    \;\leq\;
    H(P^n,Q^n)\sqrt{1-\tfrac{H^2(P^n,Q^n)}{4}}
    \;=\; \sqrt{2(1-\BC^n)}\cdot\sqrt{\tfrac{1+\BC^n}{2}}
    \;=\; \sqrt{1-\BC(P,Q)^{2n}},
\end{equation*}
where $\BC = \BC(P,Q)$ for brevity. We can therefore lower bound
\begin{align*}
    \frac{1-\TV(P^n,Q^n)}{2}&\geq \frac{1-\sqrt{1-\BC(P,Q)^{2n}}}{2} \tag{Lemma 2.3 in \citet{tsybakov2003introduction} + product formula for $\BC$} \\
    &\geq \frac{1}{4} \BC(P,Q)^{2n} \tag{using $1-\sqrt{1-x}\geq x/2$ for $x\in[0,1]$}\\
    &=\frac{1}{4}\prn{2\sqrt{\frac{s+\Delta}{2}\cdot\frac{s-\Delta}{2}} + 1-s}^{2n} \\
    &=\frac{1}{4}\prn{(1-s)+\sqrt{s^2-\Delta^2}}^{2n} \\
    &=\frac{1}{4}\exp\prn{-2n \prn{-\log\prn{(1-s)+\sqrt{s^2-\Delta^2}}}}.
\end{align*}
{Notice that because $s\in[\Delta,1/2]$, we have that 
$
(1-s)+\sqrt{s^2-\Delta^2} \in[1/2,1].
$
}
We can further bound the last term in terms of $H_3$:
\begin{align*}
    -\log\prn{(1-s)+\sqrt{s^2-\Delta^2}} &\leq 2 \prn{1-\prn{(1-s)+\sqrt{s^2-\Delta^2}}} \tag{using $-\log(x)\leq {2(1-x)}$ for {$x\in[1/2,1]$}} \\
    &= {2}\prn{s-\sqrt{s^2-\Delta^2}} \\
    &= {2}\prn{\frac{\Delta^2}{s+\sqrt{s^2-\Delta^2}}} \\
    &\stackrel{(i)}{\leq} \frac{{32}}{3}\frac{\Delta^2}{4V+\frac{4}{3}(1+\Delta)\Delta} {\leq 11}\frac{1}{2 H_3(\Delta,V)}.
\end{align*}
Here $(i)$ holds because $\frac{3}{16}\prn{4V+\frac{4}{3}(1+\Delta)\Delta} = \frac{3}{4}s+\frac{1}{4}\Delta-\frac{1}{2}\Delta^2$
and therefore the inequality is equivalent to
\begin{equation*}
    \sqrt{s^2-\Delta^2} \geq -\frac{1}{4}s+\frac{1}{4}\Delta-\frac{1}{2}\Delta^2.
\end{equation*}
Because $s\geq \Delta$, the left-hand side is always non-negative, and the right-hand side always non-positive, so $(i)$ follows.
Plugging this into the bound gives a lower bound on the average error. Since $\max\{a,b\}\ge \frac{a+b}{2}$ for any $a,b\ge 0$, we have
\begin{align*}
    \max\crl{\PP_{\cD\sim P^n}\prn{\Ihat(\cD)\neq 1},\PP_{\cD\sim Q^n}\prn{\Ihat(\cD)\neq 2}}
    \;&\geq\;
    \frac{1}{2}\PP_{\cD\sim P^n}\prn{\Ihat(\cD)\neq 1}+\frac{1}{2}\PP_{\cD\sim Q^n}\prn{\Ihat(\cD)\neq 2}\\
    &\geq\;
    \frac{1}{4}\exp\prn{-{11}\frac{n}{H_3(\Delta,V)}},
\end{align*}
which concludes the proof of \cref{thm:lower-bound}.

\subsection{Proof of \cref{cor:sufficient_conditions}}

Recall from  Proposition~\ref{thm:theta-star} that the condition \(H_{3,i} < H_{2,i}\) is a consequence to \(\theta_i > \theta^\star(\mu^\star, \Delta_i)\), where:
\[
\theta^\star(\mu^\star, \Delta_i) \;=\; \frac{N(\mu^\star, \Delta_i)}{(\mu^\star-\Delta_i)(1-\mu^\star)},
\]
with numerator \(N(\mu^\star, \Delta_i) = -\frac{1}{3}\Delta_i^2 + \Delta_i\left(\mu^\star - \frac{1}{3}\right) - \left((\mu^\star)^2 - \mu^\star + \frac{1}{4}\right)\).

\paragraph{Case 1: $\theta$-Free Sufficient Condition ($\Delta_i < \frac{3}{2}\mu^\star - 1$).}
If \(\theta^\star(\mu^\star, \Delta_i) < 0\), the condition \(\theta_i > \theta^\star\) holds for any \(\theta_i \ge 0\). This occurs when the numerator \(N(\mu^\star, \Delta_i) = -\frac{1}{3}\Delta_i^2 + (\mu^\star - \frac{1}{3})\Delta_i - (\mu^\star - \frac{1}{2})^2\) is negative. As the numerator $N(\mu^\star, \Delta_i)$ is a downward-opening parabola in $\Delta_i$, there are either no roots and $N(\mu^\star, \Delta_i)$ is always negative, or all values of $\Delta_i$ smaller than the smaller of the roots yield $N(\mu^\star, \Delta_i)< 0$. There are thus two cases: In the first, there are no roots and we are done. In the second, solving \(N(\mu^\star, \Delta_i) = 0\) for \(\Delta_i\) yields the lower root determining the positive region boundary:
\[
\Delta_- \;=\; \frac{3}{2}\mu^\star - \frac{1}{2} - \frac{1}{2}\sqrt{1 - 3(1-\mu^\star)^2} \;\ge\; \frac{3}{2}\mu^\star - \frac{1}{2} - \frac{1}{2} \;=\; \frac{3}{2}\mu^\star - 1,
\]
where we use \(\sqrt{1 - 3(1-\mu^\star)^2} \le 1\) to derive a sufficient linear lower bound.
Thus, if \(\Delta_i < \frac{3}{2}\mu^\star - 1\), then \(\Delta_i < \Delta_-\), ensuring \(N(\mu^\star, \Delta_i) < 0\) and consequently \(H_{3,i} < H_{2,i}\). Geometrically, this linear bound corresponds exactly to the tangent of the optimal boundary curve \(\Delta_-(\mu^\star)\) at \(\mu^\star=1\) (where \(\theta^\star=0\)).

\paragraph{Case 2: $\theta$-Based Sufficient Condition (\texorpdfstring{$\theta_i>2\Delta_i$}{theta>2Delta}}).
We show that \(\theta_i > 2\Delta_i \implies H_{3,i} < H_{2,i}\).
Since \(\theta_i \in [0,1]\), this condition is only feasible when \(2\Delta_i < 1\), i.e., \(\Delta_i \in (0, 1/2)\).

Recall that hardness is reduced whenever \(\theta_i > \theta^\star(\mu^\star, \Delta_i)\). Thus, it suffices to prove the uniform bound:
\[
\theta^\star(\mu^\star, \Delta_i) \le 2\Delta_i, \qquad \forall\,\mu^\star \in (\Delta_i, 1).
\]
Recall that \(\theta^\star\) is defined as the ratio \(N(\mu^\star, \Delta_i) / D(\mu^\star, \Delta_i)\), where the denominator \(D(\mu^\star, \Delta_i) = (\mu^\star - \Delta_i)(1 - \mu^\star)\) is strictly positive on the domain.
Therefore, the inequality \(\theta^\star \le 2\Delta_i\) is equivalent to \(N \le 2\Delta_i D\). We define the difference polynomial \(P(\mu^\star)\) as:
\[
P(\mu^\star) \;:=\; N(\mu^\star, \Delta_i) - 2\Delta_i(\mu^\star - \Delta_i)(1 - \mu^\star).
\]
We proceed to show that \(P(\mu^\star) \le 0\) for all admissible \(\mu^\star\).
Grouping terms  yields the quadratic:
\[
P(\mu^\star) \;=\; (2\Delta_i - 1)(\mu^\star)^2 + (1 - \Delta_i - 2\Delta_i^2)\mu^\star + \left(\frac{5}{3}\Delta_i^2 - \frac{1}{3}\Delta_i - \frac{1}{4}\right)
\]
in $\mustar$.
Since \(\Delta_i \le 1/2\), the leading coefficient is non-positive (\(2\Delta_i - 1 \le 0\)), which implies \(P\) is a concave parabola. Its value is everywhere bounded by its global maximum:
\[
\max_{\mu^\star \in [0,1]} P(\mu^\star) \;=\; -\frac{\Delta_i(1 - 2\Delta_i)(4 - 3\Delta_i)}{12}.
\]
For \(\Delta_i \in (0, 1/2)\), the is maximum strictly negative.
Thus, \(P(\mu^\star) < 0\) for all admissible \(\mu^\star\), which implies \(\theta^\star(\mu^\star, \Delta_i) < 2\Delta_i\). Consequently, the condition \(\theta_i > 2\Delta_i\) guarantees \(\theta_i > \theta^\star\) and establishes the strict hardness reduction.

\section{Algorithms}
\subsection{Standard UCB-E Algorithm Without Replacement}
\label{sec:ucb-e}

\begin{algorithm}[h!]
\caption{UCB-E}
\label{alg:ucb-e}
\begin{algorithmic}[1]
\REQUIRE Budget $n$, number of arms $K$, test instances $\mathcal{X}=\{x_1,\dots,x_L\}$, exploration parameter $a$
\STATE \textbf{Warm-up (one evaluation per arm, independent task per arm):}
\FOR{$i=1$ to $K$}
    \STATE Draw $j$ uniformly at random from $\{1,\cdots,L\}$
    \STATE Evaluate arm $i$ on $x_j$; observe reward $r \gets S_{ij}$
    \STATE $N_i \gets 1,\quad s_i \gets r$
    \STATE Mark index $j$ as evaluated for arm $i$
\ENDFOR
\FOR{$t=K+1$ to $n$}
    \FOR{$i=1$ to $K$}
        \STATE $\widehat{\mu}_i \gets \frac{s_i}{N_i}$
        \STATE $U_i \gets \widehat{\mu}_i + \sqrt{\frac{a}{N_i}}$
    \ENDFOR
    \STATE Select arm $i_t \gets \arg\max_{i \in \{1,\dots,K\}} U_i$
    \STATE Draw index $j_t$ uniformly at random from the indexes not yet evaluated by arm $i_t$
    \STATE Evaluate $i_t$ on $x_{j_t}$:
    \STATE \hspace{1.3em} $r \gets S_{i_t j_t}$
    \STATE $s_{i_t} \gets s_{i_t} + r$
    \STATE $N_{i_t} \gets N_{i_t} + 1$
    \STATE Mark $x_t$ as evaluated for arm $i_t$
\ENDFOR
\STATE Compute $\widehat{\mu}_i \gets \frac{s_i}{N_i}$ for all $i$
\RETURN $\arg\max_{i \in \{1,\dots,K\}} \widehat{\mu}_i$
\end{algorithmic}
\end{algorithm}

This variant of UCB-E from \citet{zhou2025speeding} samples test instances \emph{without replacement} per arm i.e each model is evaluated at most once on any given instance. This is a first step toward exploiting the correlation structure among LLMs, since it ensures that the empirical mean $\widehat{\mu}_i$ aggregates information from $N_i$ \emph{distinct} tasks rather than potentially redundant ones, and it sets the stage for Synchronized UCB-E (\cref{sec:smart-ucb-e}).

\subsection{UCB-E-LRF}
\label{sec:ucb-e-lrf}

\begin{algorithm}[h!]
\caption{UCB-E-LRF}
\label{alg:ucb-e-lrf}
\begin{algorithmic}[1]
\REQUIRE Budget $n$, number of arms $K$, test instances $\mathcal{X}=\{x_1,\dots,x_L\}$, low-rank factorization model $\mathcal{M}$, rank $r$, ensemble size $C$, warm-up budget $n_0$, uncertainty scaling $\eta$
\STATE \textbf{Warm-up:} Uniformly draw $n_0$ arm-instance pairs from $\{1,\dots,K\} \times \{1,\dots,L\}$
\STATE Initialize observation matrix $O \in \{0,1\}^{K \times L}$ and observed scoring matrix $S^{\text{obs}} \in ([0,1] \cup \{?\})^{K \times L}$ from these $n_0$ evaluations
\STATE $(\widehat{S},\, R) \gets \mathcal{M}(S^{\text{obs}},\, O;\, r,\, C)$
\FOR{$i = 1$ to $K$}
    \STATE $B_i \gets \frac{1}{L} \sum_{j=1}^{L} \bigl( O_{ij}\, S^{\text{obs}}_{ij} + (1 - O_{ij})\, \widehat{S}_{ij} + \eta\, R_{ij} \bigr)$
\ENDFOR
\FOR{$t = n_0 + 1$ to $n$}
    \STATE Select arm $i_t \gets \arg\max\nolimits_{k\,:\,\sum_{j} O_{kj} \neq L} B_k$
    \STATE Select instance $j_t \gets \arg\max\nolimits_{j\,:\,O_{i_t j} = 0} R_{i_t j}$
    \STATE Evaluate: $S^{\text{obs}}_{i_t j_t} \gets S_{i_t j_{t}}$
    \STATE Update: $O_{i_t j_t} \gets 1$
    \STATE $(\widehat{S},\, R) \gets \mathcal{M}(S^{\text{obs}},\, O;\, r,\, C)$
    \FOR{$i = 1$ to $K$}
        \STATE $B_i \gets \frac{1}{L} \sum_{j=1}^{L} \bigl( O_{ij}\, S^{\text{obs}}_{ij} + (1 - O_{ij})\, \widehat{S}_{ij} + \eta\, R_{ij} \bigr)$
    \ENDFOR
\ENDFOR
\RETURN $\arg\max\nolimits_{i \in \{1,\dots,K\}} \frac{1}{L} \sum_{j=1}^{L} \bigl( O_{ij}\, S^{\text{obs}}_{ij} + (1 - O_{ij})\, \widehat{S}_{ij} \bigr)$
\end{algorithmic}
\end{algorithm}

This adaptation of UCB-E relies on a low-rank approximation of the scoring matrix, fit from the partial set of observed evaluations. It yields both an imputed score matrix $\widehat{S}$ and an associated uncertainty matrix $R$ that quantifies the confidence of each predicted entry. At each round, the algorithm selects the arm maximizing an index $B_i$ that combines observed scores, imputed scores, and an uncertainty bonus, and then picks the unobserved instance with the highest uncertainty for that arm. We do not detail the estimation procedure for $\widehat{S}$ and $R$ here and refer the reader to \citet{zhou2025speeding} for the full description.

\subsection{Synchronized UCB-E}
\label{sec:smart-ucb-e}

\begin{algorithm}[h!]
\caption{Synchronized UCB-E (SyUCB-E)}
\label{alg:smart-ucb-e}
\begin{algorithmic}[1]
\REQUIRE Budget $n$, number of arms $K$, test instances $\mathcal{X}=\{x_1,\dots,x_L\}$, exploration parameter $a$
\STATE Draw a random permutation $\sigma$ of $\{1,\dots,L\}$ \hfill (global task schedule)

\STATE \textbf{Warm-up (one evaluation per arm, same first task):}
\FOR{$i=1$ to $K$}
    \STATE Evaluate arm $i$ on $x_{\sigma(1)}$; observe reward $r \gets S_{i \sigma(1)}$
    \STATE $N_i \gets 1,\quad s_i \gets r$
\ENDFOR

\FOR{$t=K+1$ to $n$}
    \FOR{$i=1$ to $K$}
        \STATE $\widehat{\mu}_i \gets \frac{s_i}{N_i}$
        \STATE $U_i \gets \widehat{\mu}_i + \sqrt{\frac{a}{N_i}}$
    \ENDFOR
    \STATE Select arm $i_t \gets \arg\max_{i \in \{1,\dots,K\}} U_i$
    \STATE Evaluate $i_t$ on its next scheduled task $x_{\sigma(N_{i_t}+1)}$:
    \STATE \hspace{1.3em} $r \gets S_{i_t \sigma(N_{i_t}+1)}$
    \STATE $s_{i_t} \gets s_{i_t} + r$
    \STATE $N_{i_t} \gets N_{i_t} + 1$
\ENDFOR

\STATE Compute $\widehat{\mu}_i \gets \frac{s_i}{N_i}$ for all $i$
\RETURN $\arg\max_{i \in \{1,\dots,K\}} \widehat{\mu}_i$
\end{algorithmic}
\end{algorithm}

As a correlation-aware variant of the classical UCB-E algorithm \citep{audibert2010best}, we propose \textbf{Synchronized UCB-E} (also referred to as SyUCB-E). This algorithm leverages the correlation structure among models by ensuring that all arms are evaluated on the same sequence of test instances.

The key modification from standard UCB-E lies in the task selection strategy: instead of selecting test instances independently and uniformly at random for each arm, Synchronized UCB-E employs a global task schedule. Specifically, a random permutation $\sigma$ of all test instances is drawn at the beginning. When an arm is selected for evaluation, it is evaluated on the test instance in this predetermined sequence corresponding to its current evaluation count. This approach ensures that arms with similar evaluation counts are evaluated on similar subsets of tasks, which can help to exploit the correlation between models' performances on shared instances. Like the standard UCB-E, Synchronized UCB-E requires specification of an exploration parameter $a$, which controls the exploration-exploitation trade-off through the upper confidence bound $U_i = \widehat{\mu}_i + \sqrt{\frac{a}{N_i}}$.

\section{Dataset Details}
\label{sec:dataset-details}

\begin{table*}[!h]
\centering
\caption{Detailed information about the 15 evaluation tasks. Tasks and Models denote the number of test instances and number of models, respectively. The datasets are sourced from HELM Lite \citep{liang2022holistic} and OpenLLM Leaderboard \citep{beeching2023open}. Best Acc and Mean Acc denote the accuracy of the best model and the mean accuracy across all models, respectively. $\Delta_2$ denotes the performance gap between the best and second-best model.
Hardness measures $H_1$, $H_2$, and $H_3$ characterize the difficulty of identifying the best model. All metrics are bounded in $[0,1]$, allowing direct application of our theoretical framework.}
\label{tab:dataset-details}
\resizebox{\textwidth}{!}{%
\begin{tabular}{lllrrrrrrrrr}
\toprule
\textbf{Dataset} & \textbf{Source} & \textbf{Metric} & \textbf{Tasks} & \textbf{Models} & \textbf{Best Acc} & \textbf{Mean Acc} & \textbf{$\Delta_2$} & \textbf{$H_1$} & \textbf{$H_2$} & \textbf{$H'_3$} & \textbf{$H_3$} \\
\midrule
arc\_challenge & OpenLLM & acc\_norm & 1,172 & 218 & 0.75 & 0.52 & 0.019 & 17568.7 & 5676.0 & 4426.2 & 3457.4  \\
bbh & OpenLLM & acc\_norm & 5,761 & 39 & 0.65 & 0.49 & 0.046 & 4048.8 & 1806.6 & 1763.8 & 1599.3 \\
gpqa & OpenLLM & acc\_norm & 1,192 & 450 & 0.44 & 0.29 & 0.025 & 32672.1 & 11561.7 & 11740.8 & 12197.4\\
ifeval & OpenLLM & prompt\_strict\_acc & 541 & 450 & 0.88 & 0.33 & 0.041 & 6213.4 & 1659.8 & 996.0 &  960.4  \\
mmlu\_pro & OpenLLM & acc & 12,032 & 48 & 0.53 & 0.31 & 0.012 & 16525.3 & 14359.2 & 14432.0 & 14727.2\\
musr & OpenLLM & acc\_norm & 756 & 450 & 0.54 & 0.40 & 0.024 & 41804.6 & 12390.7 & 13416.9 & 13749.2\\
gsm & HELM Lite & final\_num\_exact & 1,000 & 80 & 0.96 & 0.67 & 0.007 & 79251.0 & 61224.5 & 11376.0 & 3861.3\\
legalbench & HELM Lite & quasi\_exact & 2,047 & 80 & 0.70 & 0.53 & 0.004 & 121981.9 & 103461.9 & 87365.6 & 22640.7\\
math & HELM Lite & math\_equiv\_cot & 437 & 80 & 0.93 & 0.56 & 0.005 & 110595.1 & 95484.5 & 24766.5 & 10776.7\\
narrative\_qa & HELM Lite & f1 & 470 & 80 & 0.81 & 0.70 & 0.005 & 161559.9 & 96239.9 & 30047.6 & 10993.5\\
natural\_qa & HELM Lite & f1 & 2,000 & 80 & 0.65 & 0.52 & 0.008 & 47914.3 & 30948.5 &  20759.0 & 2400.1\\
wmt\_14 & HELM Lite & bleu\_4 & 5,000 & 80 & 0.23 & 0.16 & 0.003 & 335590.6 & 291296.8 & 59667.0 & 4234.8 \\
mmlu & HELM Lite & exact\_match & 567 & 82 & 0.81 & 0.61 & 0.007 & 50487.7 & 40186.1 & 25193.6 & 5005.8\\
med\_qa & HELM Lite & quasi\_exact & 1,000 & 83 & 0.87 & 0.59 & 0.008 & 43343.4 & 31250.0 & 14777.5 & 3911.3\\
commonsense & HELM Lite & exact\_match & 500 & 83 & 0.97 & 0.82 & 0.002 & 664061.7 & 500000.0 & 57096.6 & 42664.0\\
\bottomrule
\end{tabular}%
}
\end{table*}
\textbf{HELM Lite:} We use 9 available tasks from the HELM Lite benchmark suite v1.10.0 (2024-11-05).
The tasks and their subtasks are: \emph{commonsense} (OpenBookQA), \emph{gsm} (GSM8K), \emph{legalbench} (5 subsets: abercrombie, corporate\_lobbying, function\_of\_decision\_section, international\_citizenship\_questions, proa), \emph{math} (7 subjects at level 1 with chain-of-thought: algebra, counting\_and\_probability, geometry, intermediate\_algebra, number\_theory, prealgebra, precalculus), \emph{med\_qa} (MedQA), \emph{mmlu} (5 subjects: abstract\_algebra, college\_chemistry, computer\_security, econometrics, us\_foreign\_policy), \emph{narrative\_qa} (NarrativeQA), \emph{natural\_qa} (2 modes: closedbook, openbook
), and \emph{wmt\_14} (5 language pairs: cs-en, de-en, fr-en, hi-en, ru-en). For tasks with multiple subtasks, we compute model performance by averaging scores across all instances rather than averaging per-subtask averages. For each benchmark, we include all models (among a total of 83 models included in HELM) for which evaluation data on that benchmark were available for download in December 2025. 

\textbf{OpenLLM Leaderboard:} We include 6 tasks from the OpenLLM Leaderboard. The tasks and their subtasks are: \emph{ifeval} (instruction following evaluation), \emph{mmlu\_pro} (professional-level MMLU), \emph{arc\_challenge} (AI2 Reasoning Challenge), \emph{bbh} (Big-Bench Hard with 24 subtasks including boolean\_expressions, causal\_judgement, date\_understanding, disambiguation\_qa, formal\_fallacies, geometric\_shapes, hyperbaton, logical\_deduction variants, movie\_recommendation, navigate, object\_counting, penguins\_in\_a\_table, reasoning\_about\_colored\_objects, ruin\_names, salient\_translation\_error\_detection, snarks, sports\_understanding, temporal\_sequences, and tracking\_shuffled\_objects variants), \emph{gpqa} (3 subsets: diamond, extended, main), and \emph{musr} (3 subtasks: murder\_mysteries, object\_placements, team\_allocation). For each benchmark we select all models from official providers (from a total of 470) for which complete evaluation data across all subtasks was available in December 2025. 

\textbf{Model Lists:} Data from both sources includes hundreds of evaluated models; rather than listing all models in the paper, we provide a full list of evaluated models together with evaluation parameters and data sources at \url{https://github.com/zifanlyu/llm-bandits-sysrs/blob/main/datasets/MODEL_CANDIDATES_AND_EVALUATION_PARAMETERS_BY_DATASET.md}.

\section{Implementation Details}
\label{sec:implementation-details}

\subsection{Budget Levels}
We evaluate algorithm performance across a range of budget percentages relative to the total number of evaluations needed for exhaustive evaluation (i.e., $K \times L$). 

\paragraph{Budget Grid for SySRs, SyUCB-E, UCB-E, US and SyUS.}
For all algorithms besides UCB-E-LRF and UCB-E-LRF(No Warm-up), we evaluate performance at every integer percentage $1\%, 2\%, \ldots, 100\%$, augmented with sub-integer grid points $\{1.25\%, 1.5\%, 1.75\%, 2.5\%, 3.5\%, 4.5\%, 5.5\%\}$ for finer resolution at low budgets, yielding 107 budget levels in total.

\paragraph{Budget Grid for UCB-E-LRF and UCB-E-LRF(No Warm-up).}
UCB-E-LRF experiments are evaluated on a sparser grid of 38 budget levels: sub-integer steps $\{1.25\%, 1.5\%, 1.75\%, 2.5\%, 3.5\%, 4.5\%, 5.5\%\}$, every $1\%$ from $6\%$ to $10\%$, every $2\%$ from $10\%$ to $20\%$, and then every $5\%$ from $25\%$ to $100\%$.

\paragraph{CB Computation.}
For algorithms evaluated on the finer grid (all except UCB-E-LRF variants), the X\% Confidence Budget (CB) is the minimum budget percentage at which the identification accuracy reaches X\% and is maintained at all subsequent evaluated budget points. Since some experiments are evaluated on the sparser grid, we compute the CB of every UCB-E-LRF and UCB-E-LRF(No Warm-up) experiment by rounding down to the nearest evaluated budget point on the sparser grid. This approach provides a favorable comparison for UCB-E-LRF and UCB-E-LRF(No Warm-up), as it may underestimate the true CB (i.e., the actual threshold could occur between two evaluated points). Despite this favorable treatment of some baselines, SySRs still demonstrates superior performance across most datasets and thresholds (see \cref{tab:perfect-identification}).

\subsection{Experimental Settings for SySRs, SyUCB-E, UCB-E, and SR}
We perform $1000$ independent runs for each algorithm at each budget level to obtain robust estimates of the probability of identification. For UCB-E and SyUCB-E, we conduct a sensitivity analysis with respect to the exploration hyperparameter $a \in \{0.1, 1.0, 10.0, 100.0\}$.

\subsection{UCB-E-LRF Hyperparameters and Settings}
\label{sec:ucb-elf-hyperparams}
For UCB-E-LRF, we use the implementation from the original authors' codebase \citep{zhou2025speeding}. Based on their findings that the algorithm's performance is not highly sensitive to hyperparameter choices, we adopt their default hyperparameter settings.

Due to the computational cost of the low-rank factorization approach, we run only $100$ independent runs (compared to $1000$ runs for other algorithms). We fix the rank parameter to $r = 1$ and use the following hyperparameters: ensemble size of 64, uncertainty scaling $\eta = 5.0$, dropout rate $\text{drop} = 0.05$, ALS (Alternating Least Squares) regularization $\lambda = 0.1$, 10 ALS iterations per update, batch size of 32 samples per UCB iteration, and warmup percentage $T_0 = 5\%$ of total matrix entries. 

To evaluate the impact of the warmup phase, we also implement UCB-E-LRF(No Warm-up) with $T_0 = 0\%$ of total matrix entries while keeping all other hyperparameters unchanged.

\subsection{Summary of Evaluated Algorithms and Their Characteristics.}
We summarize the proposed and evaluated algorithms and their key characteristics in \cref{tab:algorithm-summary}, including their origin, whether they utilize synchronized evaluation, hyperparameter settings, and computational overhead per run.
\begin{table*}[h]
\centering
\small
\caption{Summary of evaluated algorithms and their characteristics.}
\label{tab:algorithm-summary}
\begin{tabular}{lllll}
\toprule
\textbf{Algorithm} & \textbf{Origin} & \textbf{Sync.} & \textbf{Hyperparameter} & \textbf{Overhead} \\
\midrule
SySRs & This work (\cref{alg:csr}) & Yes & Free & Seconds \\
SR & \citet{audibert2010best} & No & Free & Seconds \\
SyUCB-E & This work (\cref{sec:smart-ucb-e}) & Yes & $a \in \{0.1, 1, 10, 100\}$ & Minutes \\
UCB-E & \citet{audibert2010best} & No & $a \in \{0.1, 1, 10, 100\}$ & Minutes \\
UCB-E-LRF & \citet{zhou2025speeding} & N/A & \cref{sec:ucb-elf-hyperparams} & Hours \\
UCB-E-LRF(NW) & Modified from \citet{zhou2025speeding} & N/A & \cref{sec:ucb-elf-hyperparams} & Hours \\
SyUS & Baseline & Yes & Free & Seconds \\
US & Baseline & No & Free & Seconds \\
\bottomrule
\multicolumn{5}{l}{\footnotesize Sync. = Synchronized evaluation; Overhead = Computational overhead per run.} \\
\multicolumn{5}{l}{\footnotesize NW = No Warm-up; N/A = Not applicable (uses matrix completion, not direct synchronization).}
\end{tabular}
\end{table*}

\section{Full Experimental Results}
\label{sec:all-dataset-plots}

\subsection{Mean Error Rate and Worst-Case Confidence Budget Across Datasets}
In this section, we present aggregate performance metrics across all 15 datasets to provide a holistic view of each algorithm's effectiveness. \cref{fig:mean_all} a illustrates the error rate $\hat{\epsilon}_b$ as a function of the budget percentage, averaged over all datasets. This plot highlights how quickly each algorithm converges to accurate model identification as the evaluation budget increases. \cref{fig:mean_all} b shows the worst-case Confidence Budget (CB) required to achieve specified identification accuracy levels across all datasets. This visualization emphasizes the robustness of each algorithm in the most challenging scenarios.
\begin{figure}[!h]
    \centering
    \includegraphics[width=1.0\linewidth]{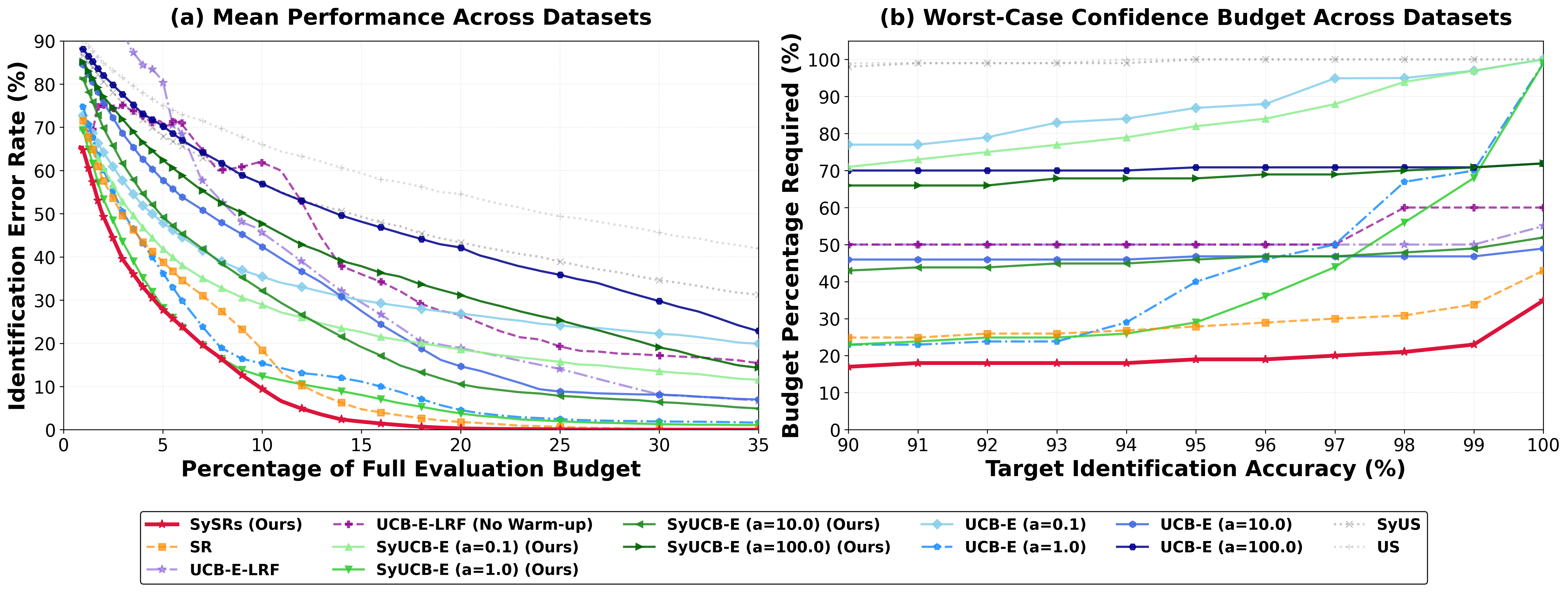}
    \caption{\textbf{a)}: Error rate using a given percentage of model/query pairs, averaged over 15 LLM benchmarks. \textbf{b)} Percentage of model/query pairs required to identify the best model with a given accuracy, worst case over 15 LLM benchmarks. }
    \label{fig:mean_all}
\end{figure}

\subsection{Individual Plots for All Datasets}
In this section, we provide detailed plots for all 15 datasets used in our experiments, illustrating the performance of each algorithm. For each dataset, we present two key visualizations: the left panel shows the estimated error rate $\hat{\epsilon}_b$ as a function of the budget percentage $b$, and the right panel shows the Confidence Budget (CB) required to achieve specified identification accuracy levels. 

\begin{figure*}[h!]
\centering
\begin{subfigure}[t]{\textwidth}
\centering
\includegraphics[width=\textwidth]{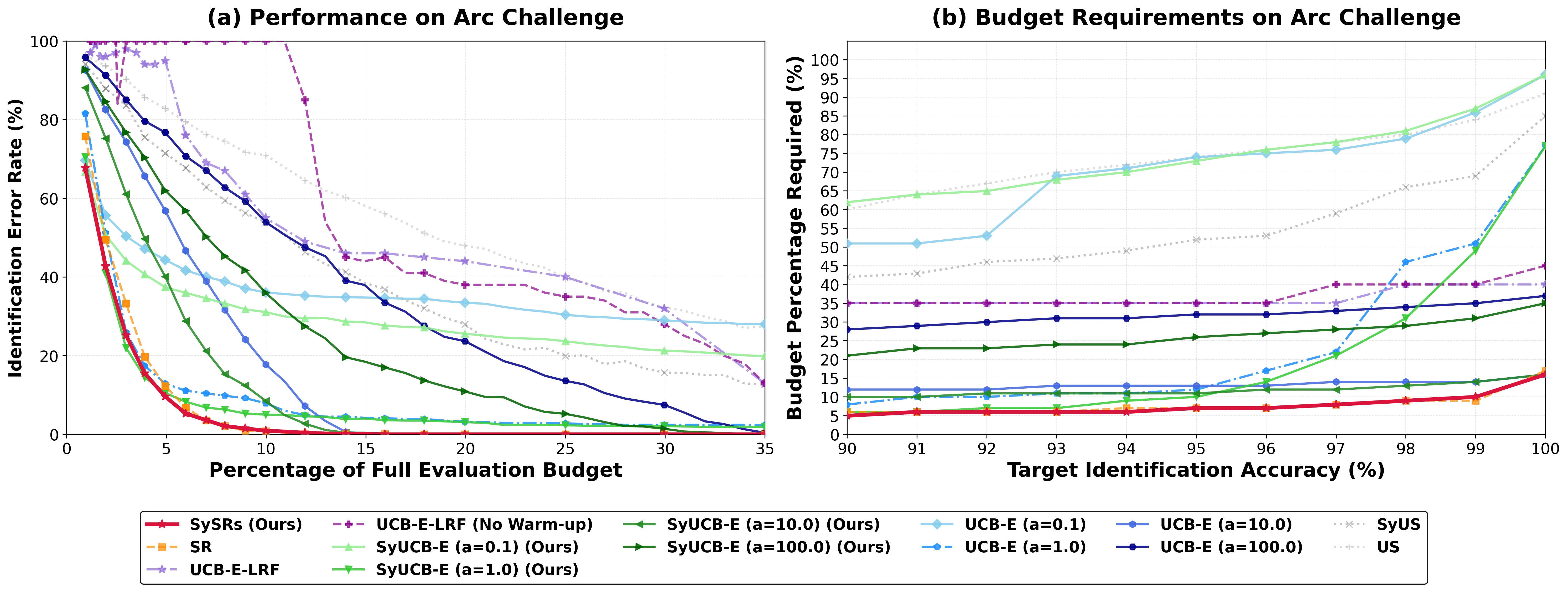}
\caption{ARC Challenge dataset results.}
\label{fig:arc-challenge}
\end{subfigure}

\vspace{0.8em}

\begin{subfigure}[t]{\textwidth}
\centering
\includegraphics[width=\textwidth]{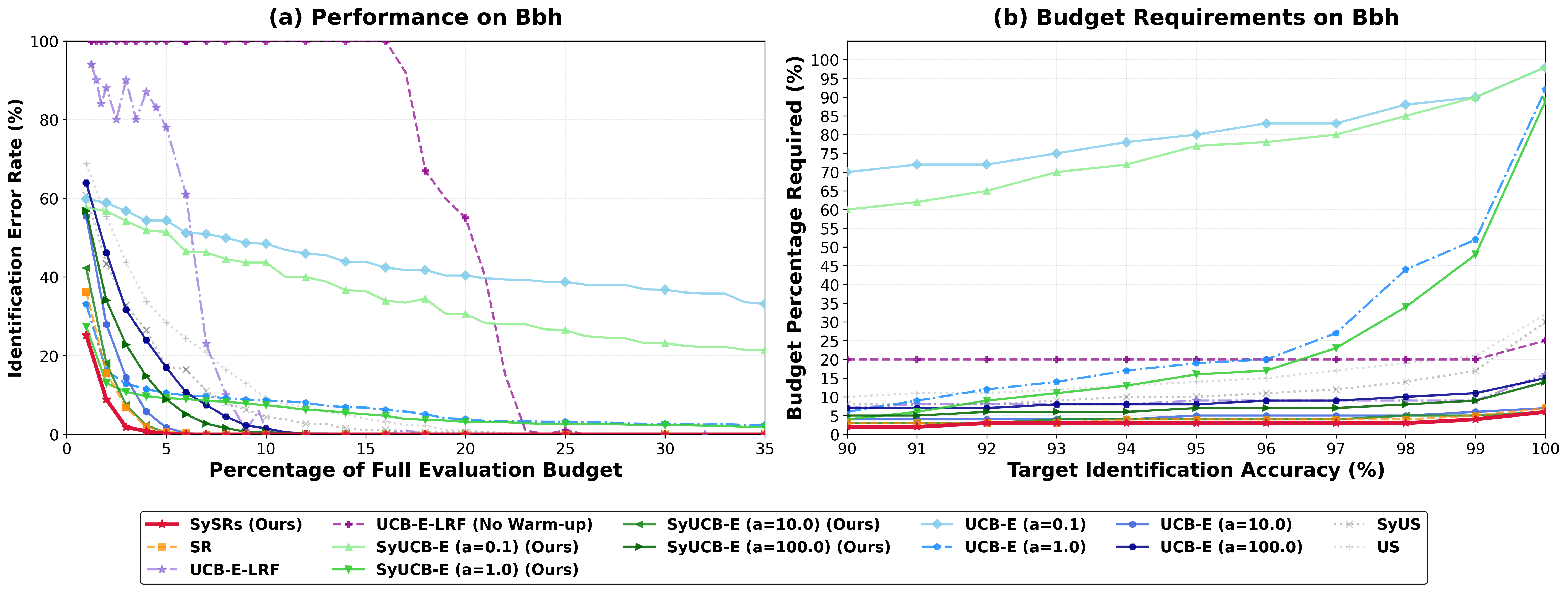}
\caption{BIG-Bench Hard dataset results.}
\label{fig:bbh}
\end{subfigure}
\caption{Individual dataset results (Page 1/8).}
\end{figure*}

\begin{figure*}[p]
\centering
\begin{subfigure}[t]{\textwidth}
\centering
\includegraphics[width=\textwidth]{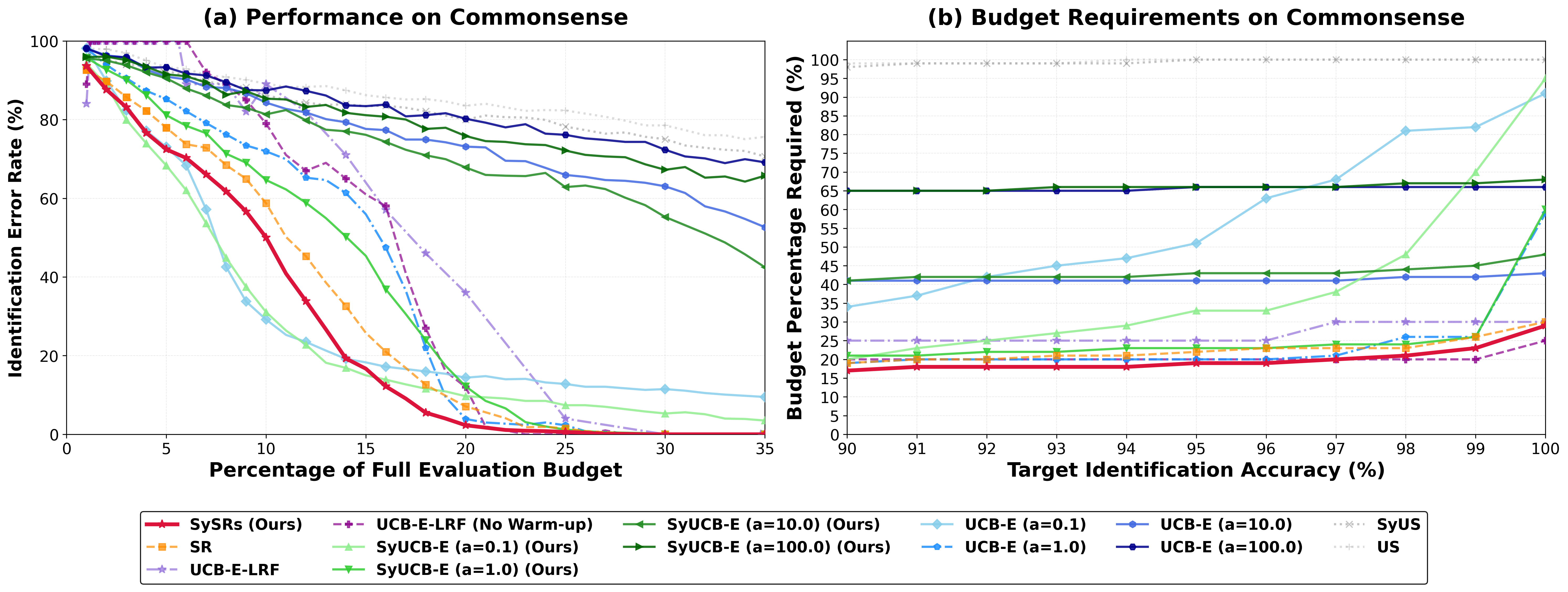}
\caption{Commonsense dataset results.}
\label{fig:commonsense}
\end{subfigure}

\vspace{0.8em}

\begin{subfigure}[t]{\textwidth}
\centering
\includegraphics[width=\textwidth]{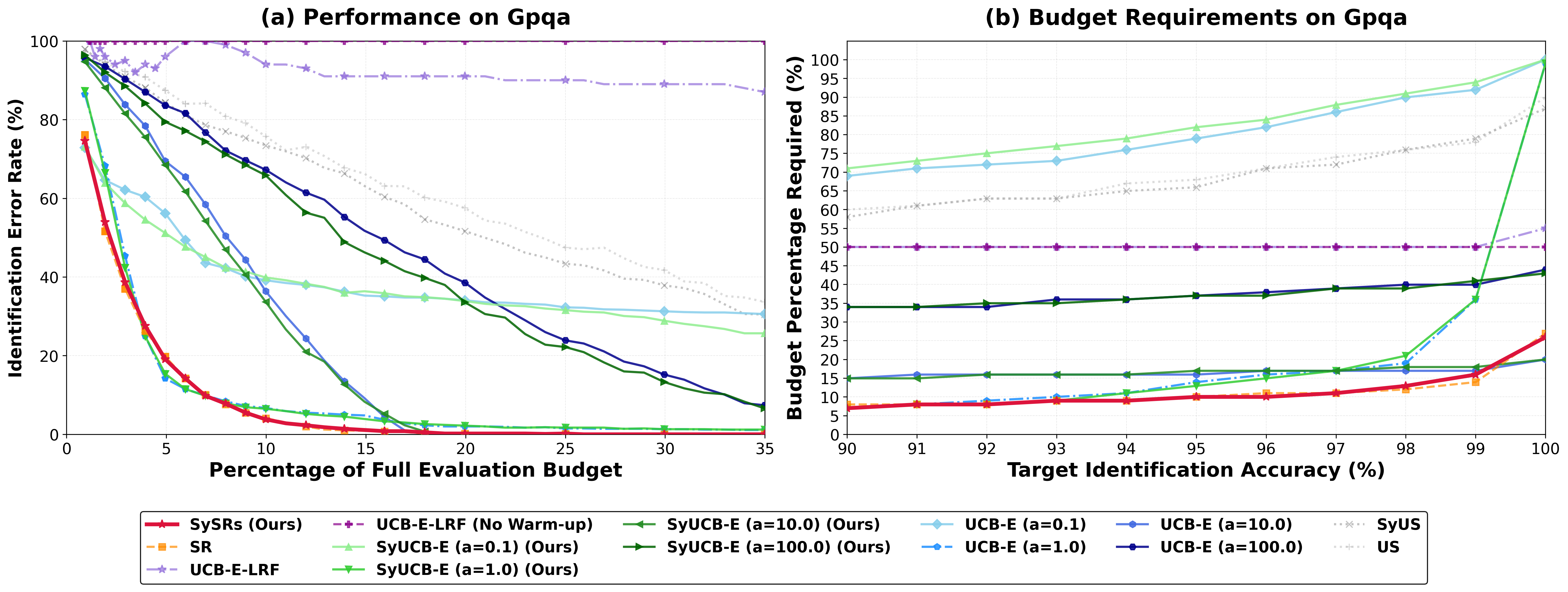}
\caption{GPQA dataset results.}
\label{fig:gpqa}
\end{subfigure}
\caption{Individual dataset results (Page 2/8).}
\end{figure*}

\begin{figure*}[p]
\centering
\begin{subfigure}[t]{\textwidth}
\centering
\includegraphics[width=\textwidth]{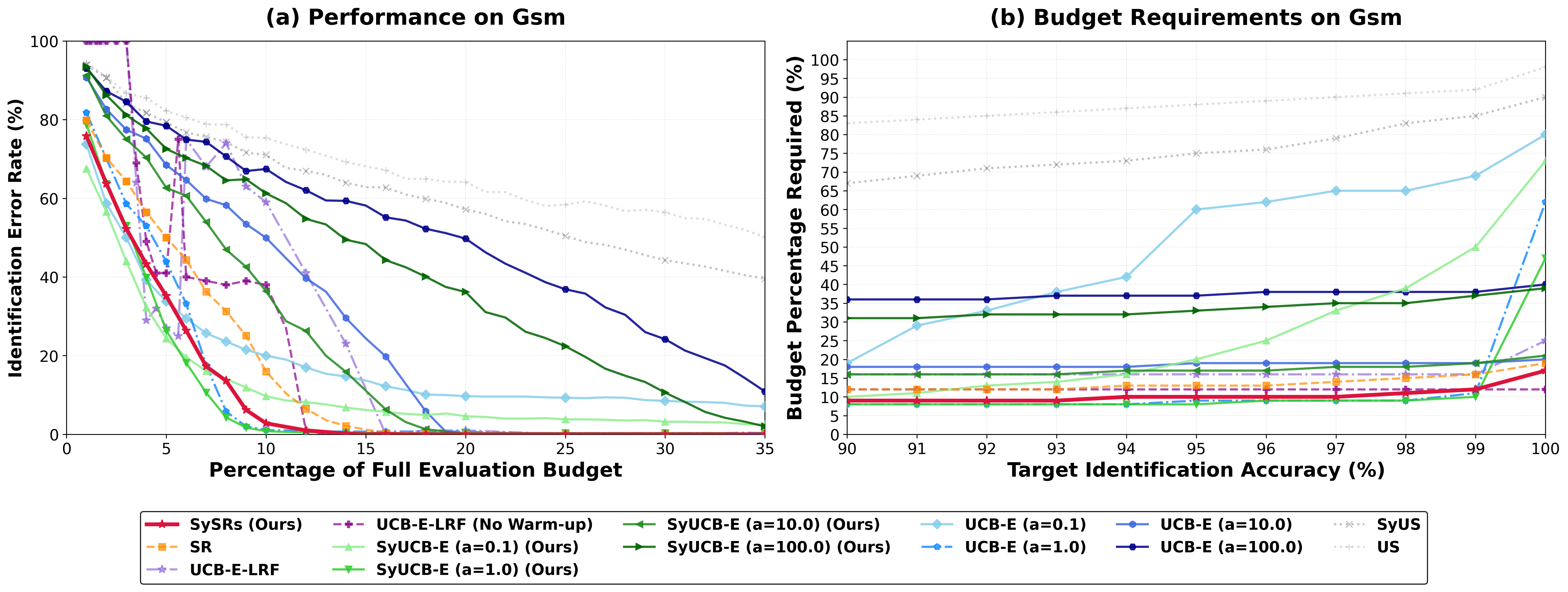}
\caption{GSM dataset results.}
\label{fig:gsm}
\end{subfigure}

\vspace{0.8em}

\begin{subfigure}[t]{\textwidth}
\centering
\includegraphics[width=\textwidth]{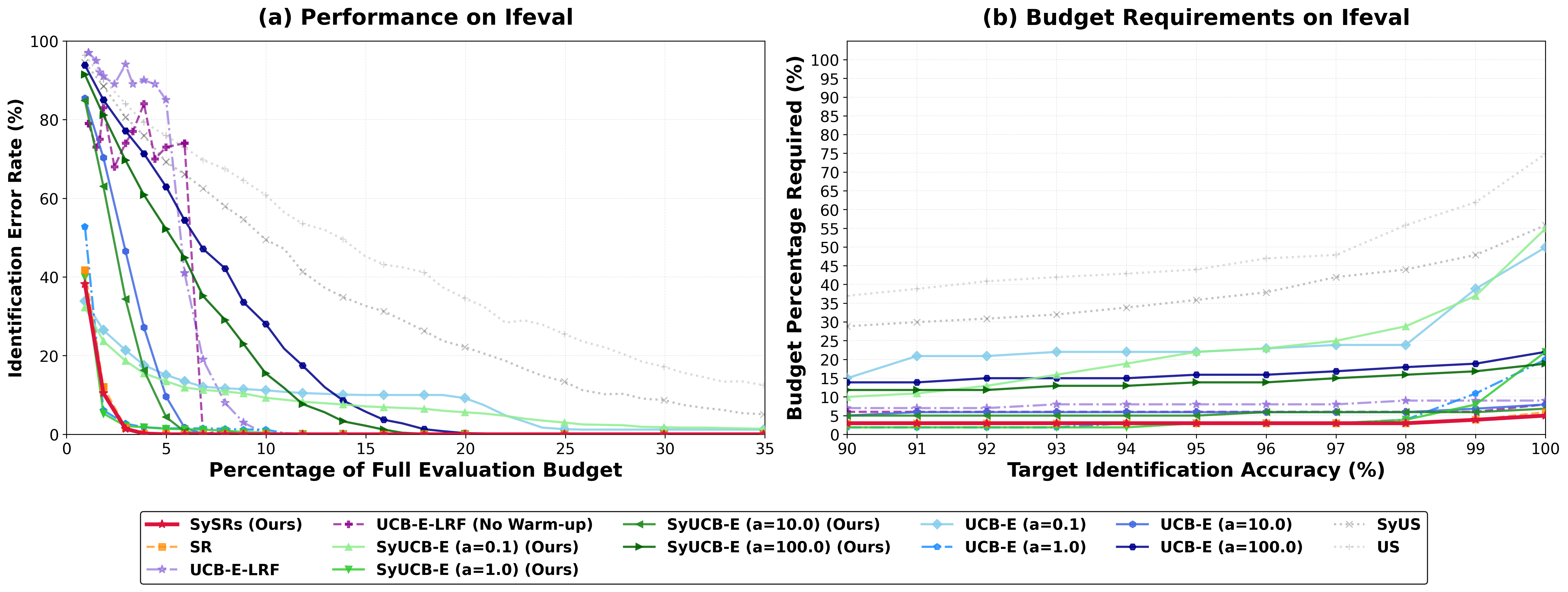}
\caption{IFEval dataset results.}
\label{fig:ifeval}
\end{subfigure}
\caption{Individual dataset results (Page 3/8).}
\end{figure*}

\begin{figure*}[p]
\centering
\begin{subfigure}[t]{\textwidth}
\centering
\includegraphics[width=\textwidth]{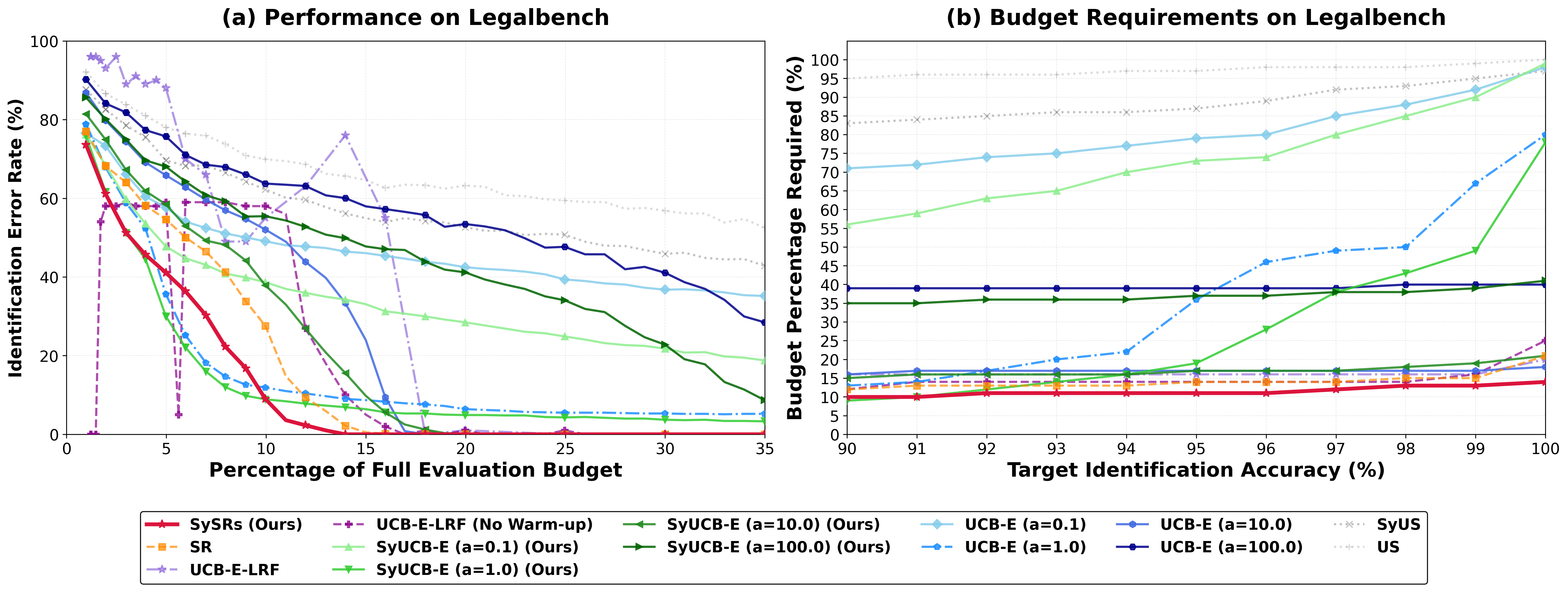}
\caption{LegalBench dataset results.}
\label{fig:legalbench}
\end{subfigure}

\vspace{0.8em}

\begin{subfigure}[t]{\textwidth}
\centering
\includegraphics[width=\textwidth]{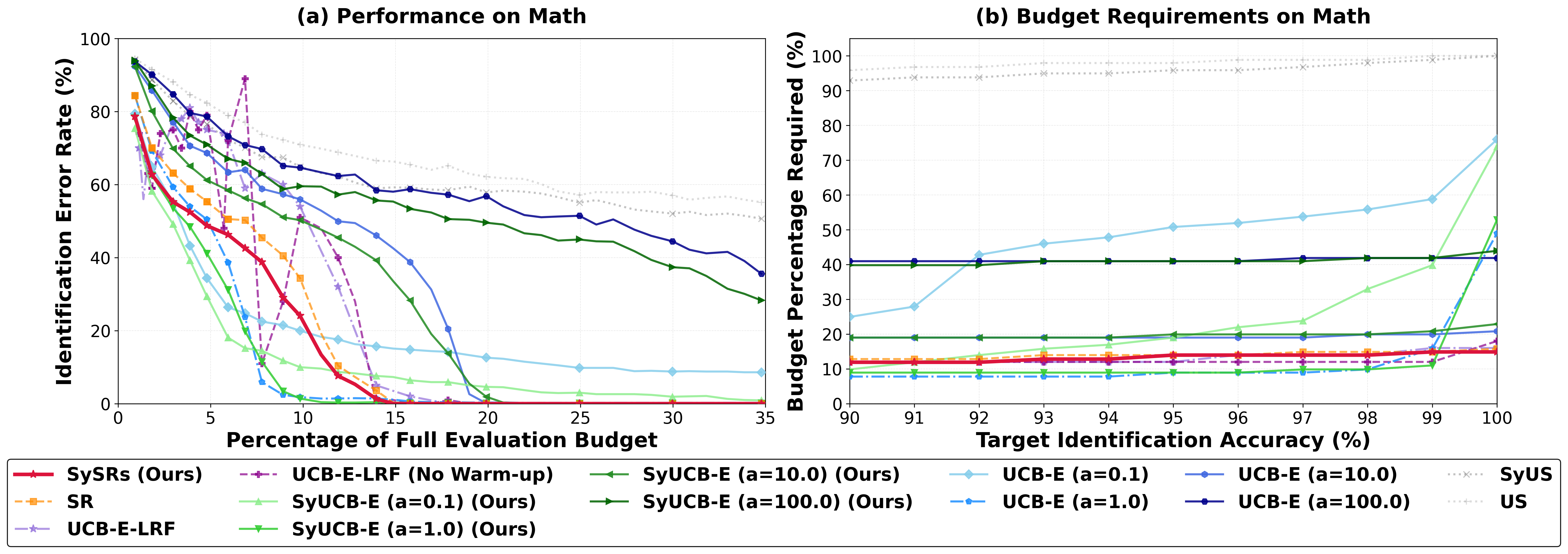}
\caption{MATH dataset results.}
\label{fig:math}
\end{subfigure}
\caption{Individual dataset results (Page 4/8).}
\end{figure*}

\begin{figure*}[p]
\centering
\begin{subfigure}[t]{\textwidth}
\centering
\includegraphics[width=\textwidth]{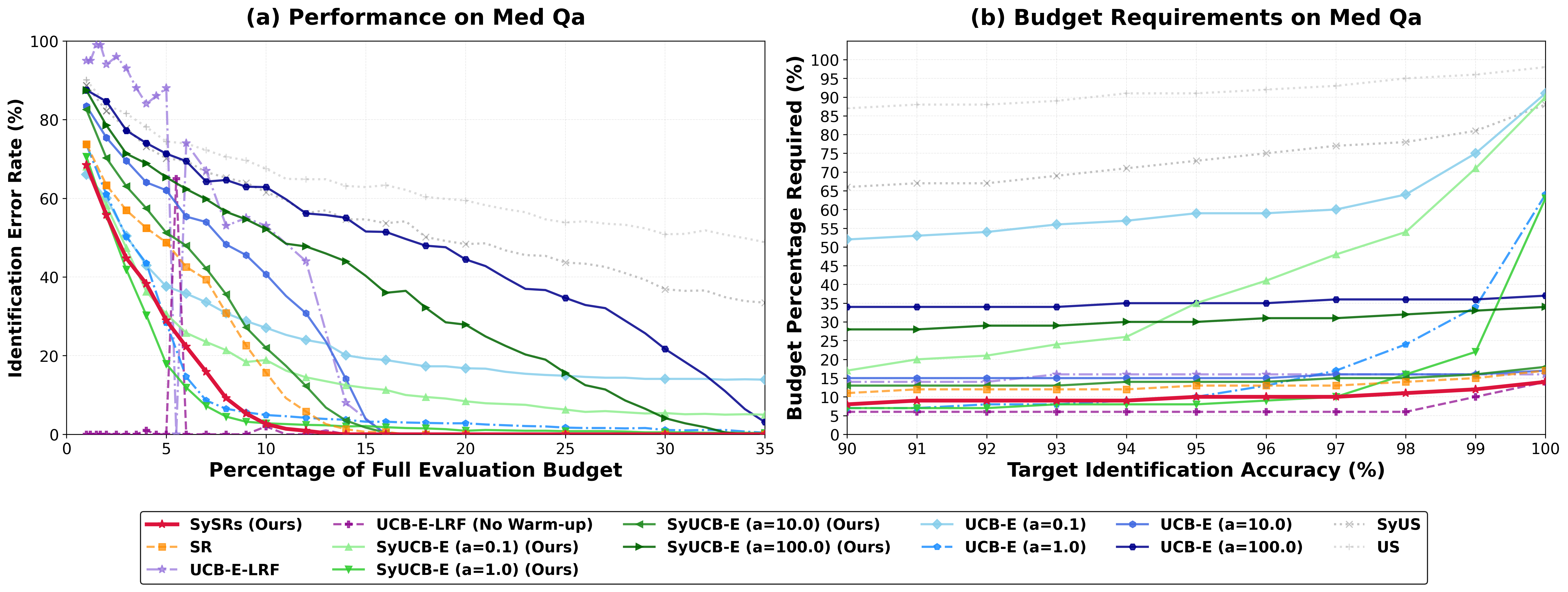}
\caption{MedQA dataset results.}
\label{fig:medqa}
\end{subfigure}

\vspace{0.8em}

\begin{subfigure}[t]{\textwidth}
\centering
\includegraphics[width=\textwidth]{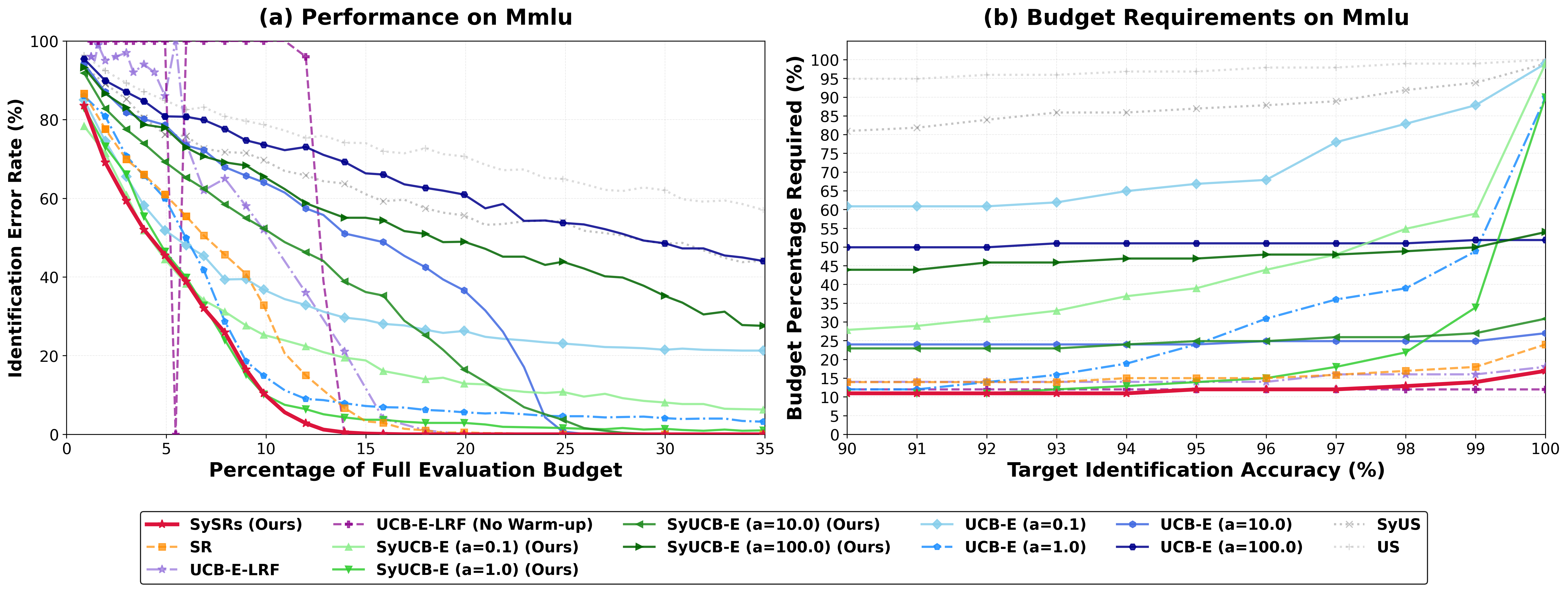}
\caption{MMLU dataset results.}
\label{fig:mmlu}
\end{subfigure}
\caption{Individual dataset results (Page 5/8).}
\end{figure*}

\begin{figure*}[p]
\centering
\begin{subfigure}[t]{\textwidth}
\centering
\includegraphics[width=\textwidth]{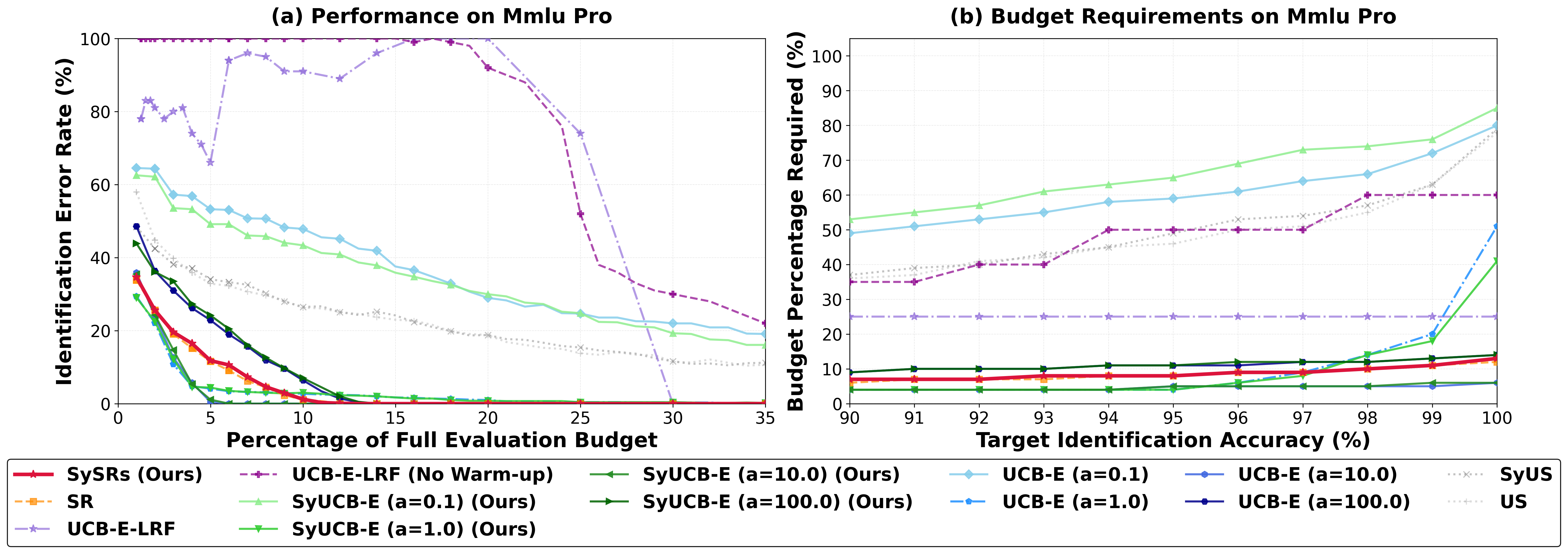}
\caption{MMLU-Pro dataset results.}
\label{fig:mmlu-pro}
\end{subfigure}

\vspace{0.8em}

\begin{subfigure}[t]{\textwidth}
\centering
\includegraphics[width=\textwidth]{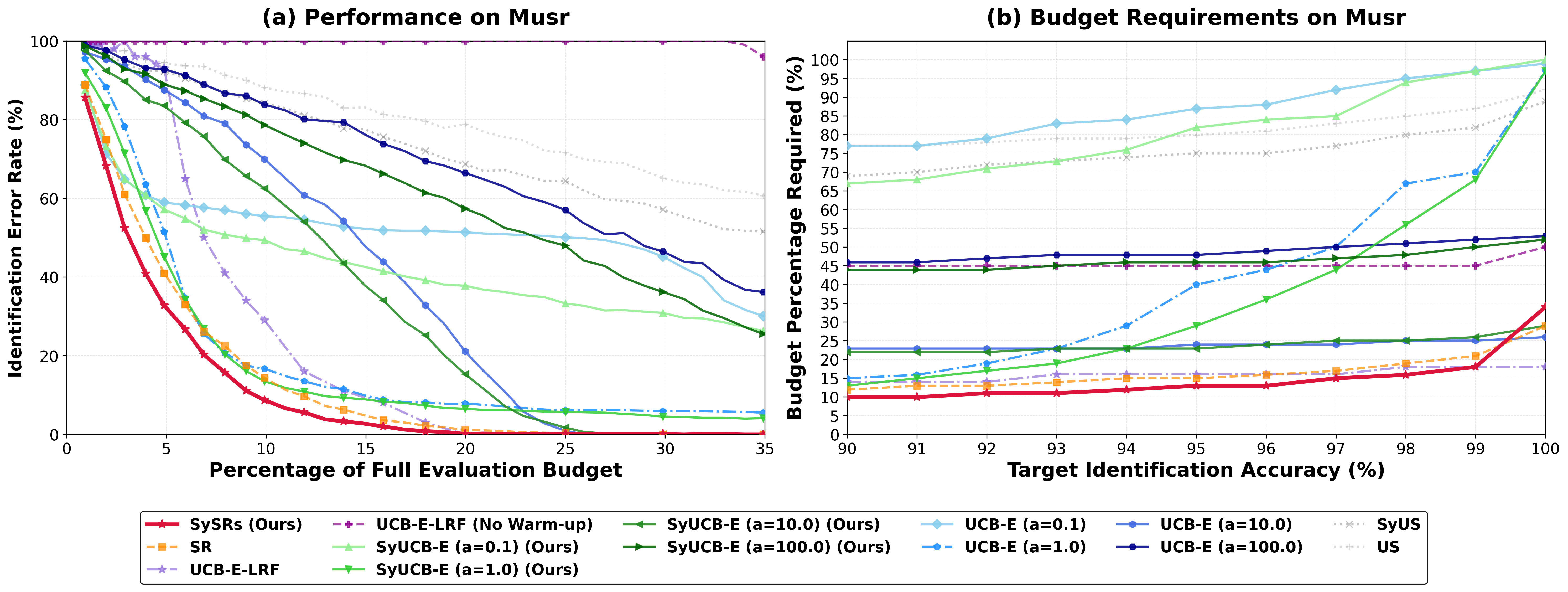}
\caption{MuSR dataset results.}
\label{fig:musr}
\end{subfigure}
\caption{Individual dataset results (Page 6/8).}
\end{figure*}

\begin{figure*}[p]
\centering
\begin{subfigure}[t]{\textwidth}
\centering
\includegraphics[width=\textwidth]{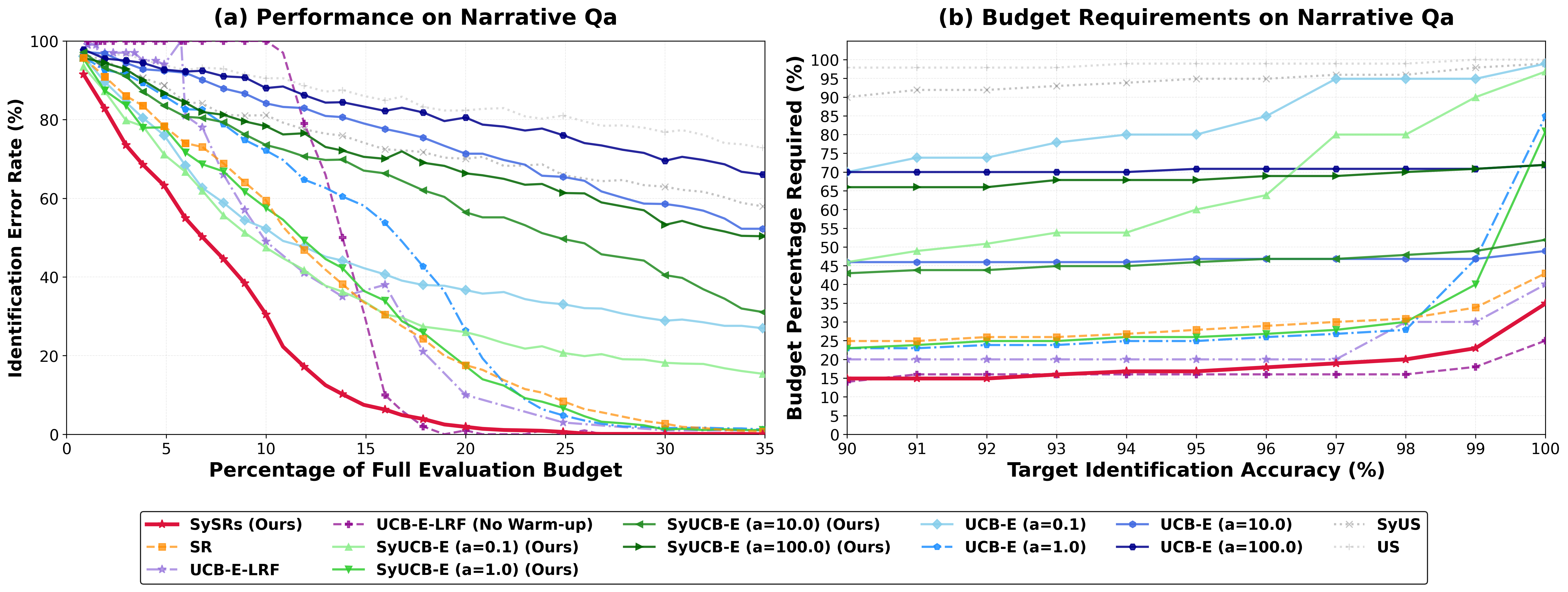}
\caption{NarrativeQA dataset results.}
\label{fig:narrativeqa}
\end{subfigure}

\vspace{0.8em}

\begin{subfigure}[t]{\textwidth}
\centering
\includegraphics[width=\textwidth]{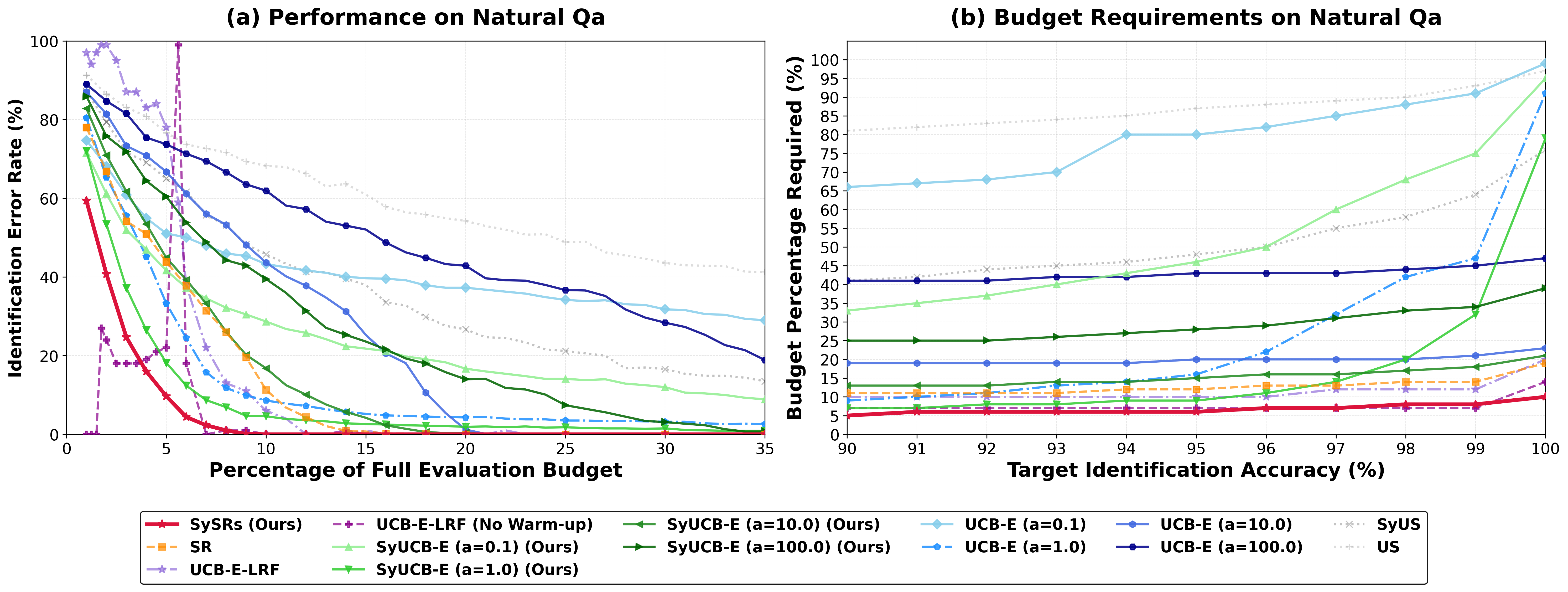}
\caption{NaturalQA dataset results.}
\label{fig:naturalqa}
\end{subfigure}
\caption{Individual dataset results (Page 7/8).}
\end{figure*}

\begin{figure*}[h!]
\centering
\begin{subfigure}[t]{\textwidth}
\centering
\includegraphics[width=\textwidth]{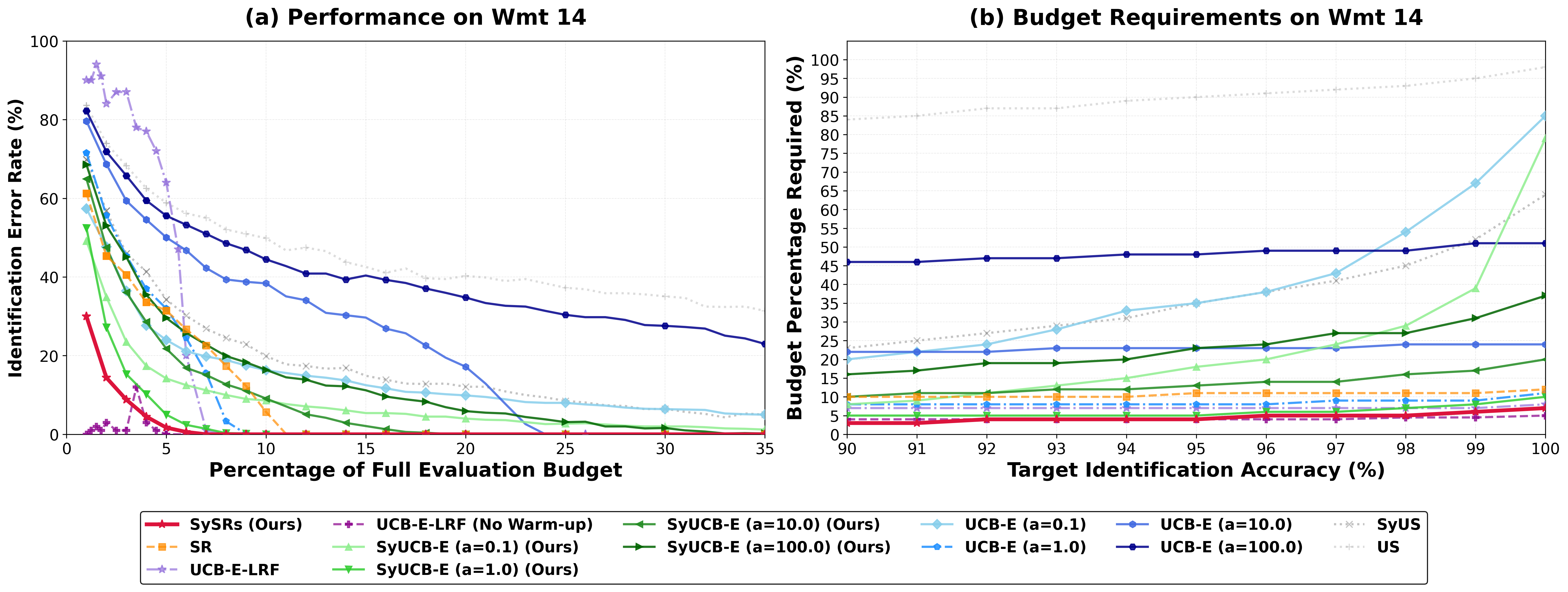}
\caption{WMT-14 dataset results.}
\label{fig:wmt14}
\end{subfigure}
\caption{Individual dataset results (Page 8/8).}
\end{figure*}

\newpage\section{Subset Selection Experimental Details}
\label{sec:subset-selection}

We provide additional details on the subset selection experiments mentioned in the main paper, including the experimental setup and implementation details for each subset selection method. 

\subsection{Experimental Setup}
Let $\S \in [0,1]^{K \times L}$ be the full score matrix over $K$ models and $L$ questions. The ground-truth best arm is $\istar \in \arg\max_i \mu_i$ where $\mu_i = \frac{1}{L}\sum_{j=1}^{L} S_{ij}$. At budget fraction $b$, the total query budget is $\lfloor b \cdot KL \rfloor$ and the number of anchor questions is $k = \lfloor b L \rfloor$. The subset selection methods are evaluated at the budget levels $b \in \{5\%, 10\%, 15\%, 20\%, 25\%, 30\%, 35\%\}$. For each method and budget, we first select a subset of $k$ anchor questions using that method, then evaluate all $K$ models on the selected anchors, compute an estimate $\hat{\mu}_i$ for each model, and predict the best arm as $\hat{\imath}$. 

The clustering-based methods are each evaluated over 100 independent runs; the only source of randomness across runs is the clustering initialisation seed, while the score matrix $\S$ and all derived parameters (e.g., correlations, IRT parameters) are fixed. MetaBench is fully deterministic given $\S$, so a single run suffices, and its best-arm identification rate is either 0 or 1 on a given dataset and budget.

\subsection{Subset Selection via Anchor Point Weighted Score from \citet{vivek2024anchor}}
We replicate the \textbf{Anchor Point Weighted Score} method from~\citet{vivek2024anchor}. First, we compute the inter-question Pearson correlation matrix $C \in \mathbb{R}^{L \times L}$ with
\[
  C_{ji} = \operatorname{corr}(\S_{: j},\, \S_{: i}),
\]
and the distance matrix $D = \mathbf{1} - C$. We then run k-medoids on $D$ to obtain $k$ representative questions $\{q_1,\ldots,q_k\}$. Each question $q$ in the full benchmark is assigned to its most-correlated medoid,
\[
  a_q = \arg\max_{l \in \{1,\dots,k\}}\; C_{q_l,\, q},
\]
and the weight of the anchor item $q_l$ is set to the fractional cluster size of its assigned questions,
\[
  w_l = \frac{|\{q : a_q = l\}|}{L}.
\]
We then predict the score of model $i$ as the weighted average of its scores on the anchor questions,
\[
  \hat{\mu}_i = \sum_{l=1}^{k} w_l \cdot S_{i q_l}.
\]

The only randomness across runs comes from the seed used by the k-medoids algorithm.

\subsection{Subset Selection via K-means from \citet{polo2024tinybenchmarks}}
\label{sec:k_means}

We replicate two subset selection methods from~\citet{polo2024tinybenchmarks} that differ only in how each question is embedded before clustering. Both methods run K-means with $k$ clusters on the question embeddings $\{x_j\}_{j=1}^{L}$ and use the question nearest to each cluster centre as anchors,
\[
  q_l^\star = \arg\min_{q \in [L]}\; \|x_q - c_l\|_2,
\]
assign weight $w_l = |\mathrm{cluster}_l| / L$ (fractional cluster size), and predict the score for model $i$ as
\[
  \hat{\mu}_i = \sum_{l=1}^{k} w_l \cdot S_{i q_l^\star}.
\]
Randomness across runs comes solely from the K-means initialisation seed.

\paragraph{IRT-based subset selection.}
A multidimensional 2PL IRT model is fitted once to $\S$ and cached. The model assigns each question $j$ a discrimination vector $\alpha_j \in \mathbb{R}^{D-1}$ and scalar difficulty $\beta_j \in \mathbb{R}$, with item response probability
\[
  P_{ij} = \sigma\!\left(\alpha_j^\top \theta_i - \beta_j\right),
\]
where $\theta_i \in \mathbb{R}^{D-1}$ is the latent ability of model $i$. The IRT is fitted with hierarchical priors via MAP estimation for 2000 epochs. Each question is embedded as $x_j = (\alpha_j,\, \beta_j) \in \mathbb{R}^{D}$. We evaluate $D \in \{2, 5, 10, 15\}$, corresponding to $D-1$ discrimination dimensions.

\paragraph{Correctness-based subset selection.}
Each question $j$ is embedded directly as its score vector across all $K$ models,
\[
  x_j = \S_{: j} \in \mathbb{R}^{K},
\]
i.e.\ the column of $\S$ recording each model's binary score on question $j$. No IRT fitting is required.

\subsection{MetaBench from \citet{kipnis2025metabench}}
MetaBench~\citep{kipnis2025metabench} selects a subset of $k$ questions via a Fisher-information quantile criterion on a fitted 2PL IRT model, then predicts each model's full-bank mean score using a calibrated generalized additive model (GAM). The procedure is entirely deterministic given~$\S$.

\paragraph{Step 1: Fit a 2PL IRT model.}
A 2PL IRT model is fitted to $\S$ with fixed initialisations, yielding discrimination $\alpha_j \in \mathbb{R}$ and difficulty $\beta_j \in \mathbb{R}$ for each question $j$, and ability $\theta_i \in \mathbb{R}$ for each model $i$.

\paragraph{Step 2. Subset Selection based on Fisher Information}
The $K$ fitted abilities are partitioned into $k$ equal-probability quantile bins $\mathcal{B}_t = \{i : \hat\theta_i\in[\tau_t,\tau_{t+1})\}$, $t=1,\ldots,k$. For $t=1,\ldots,\min(k,K)$, the anchor for bin $t$ is the not-yet-selected question with the highest average Fisher information over models in that bin,
\[
  q_t = \arg\max_{j\notin Q_{t-1}}\;\frac{1}{|\mathcal{B}_t|}\sum_{i\in\mathcal{B}_t} a_j^2\,P_{ij}(1-P_{ij}),
\]
where $Q_{t-1}=\{q_1,\ldots,q_{t-1}\}$. Any remaining slots ($t>K$, since each bin yields at most one distinct item) are filled by the question with the highest global average Fisher information:
\[
  q_t = \arg\max_{j\notin Q_{t-1}}\;\frac{1}{K}\sum_{i=1}^{K} a_j^2\,P_{ij}(1-P_{ij}).
\]
The final subset is $Q^*=\{q_1,\ldots,q_k\}$.

\paragraph{Step 3: Calibrate a GAM and predict.}
A per-subset ability $\hat\theta_i^{Q^*}$ is re-estimated on $Q^*$ only. A two-term GAM
\[
  \hat\mu_i = f_1\!\bigl(\hat\theta_i^{Q^*}\bigr) + f_2\!\bigl(\bar{S}_i^{Q^*}\bigr) + c
\]
is fitted by penalised least squares with the full-bank mean $\mu_i$ as the response, where $\bar{S}_i^{Q^*}$ is model $i$'s mean score on $Q^*$ and $f_1,f_2$ are cubic spline smooths. The predicted best model is $\hat\imath = \arg\max_i \hat\mu_i$.

\subsection{Results and Comparison to SySRs}
We fit all subset selection methods on the full matrix $\S$. There is no train/test split. This setup gives a large advantage to the subset selection methods, as only some entries of $\S$ would be known \emph{a priori} in practice. Indeed, knowing the full matrix trivializes the model selection problem, as all models' average performance would already be known with high precision. Despite this advantage, we find that for model selection, the subset selection methods still fail to match the performance of SySRs. \cref{fig:subset-selection-results} shows the mean identification error rate across all 15 datasets. In particular, the identification error rate of SySRs, using only 5\% of the budget is roughly on par with the error rate of the best subset selection method using 30\% of the budget. This is perhaps unsurprising, as subset selection methods are not built for model selection specifically, and our experimental setup still requires them to evaluate all models on the full subset selection. In principle, subset selection methods could be combined with the bandit approach, but this would forego theoretical guarantees for empirical gains that are relatively minor in many settings \cite{yauney2026}.

\begin{figure}[h]
    \centering
    \includegraphics[width=\textwidth]{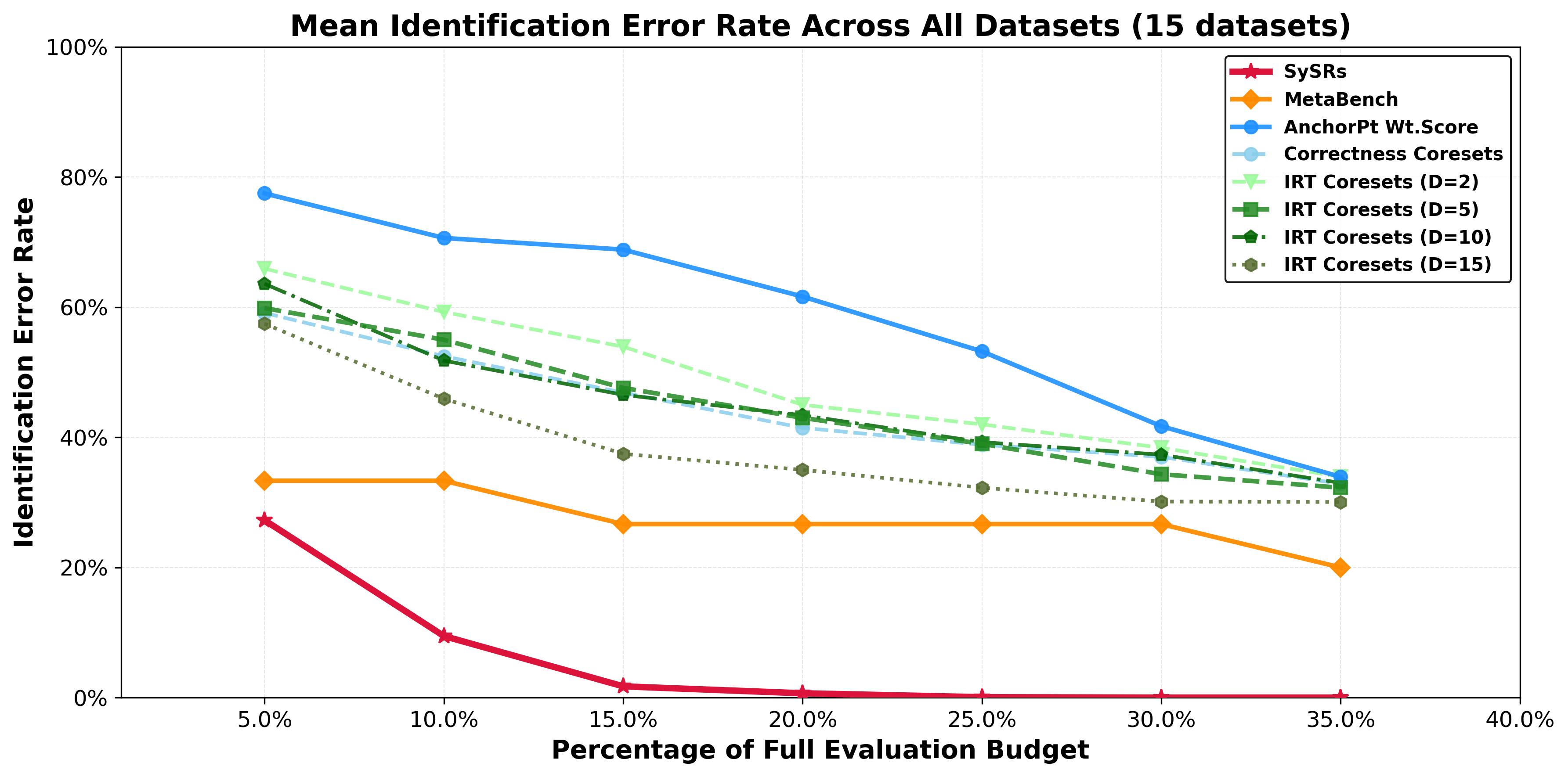}
    \caption{Mean identification error rate across all 15 datasets as a function of the evaluation budget, comparing SySRs to subset selection methods. All subset selection methods have access to the full score matrix~$\S$ before selection. Despite this advantage, all subset selection methods exhibit substantially higher error rates than SySRs across the entire budget range. }
    \label{fig:subset-selection-results}
\end{figure}

\newpage\section{Top-\texorpdfstring{$m$}{m} Ranking and Identification}
\label{sec:top-m-ranking}

We now extend the evaluation framework of \cref{sec:experiments} beyond best-arm identification to the \emph{top-$m$ identification and ranking problem}, where the goal is to recover the set of the $m$ strongest models given a fixed budget.

\subsection{Problem Formulation}

Recall the score matrix $\S \in [0,1]^{K \times L}$ and mean performance $\mu_i = \frac{1}{L}\sum_{j=1}^L S_{ij}$.
Based on the mean performance on the benchmark, the models are ordered and indexed in decreasing order with $\mu_{(1)} \ge \mu_{(2)} \ge \cdots \ge \mu_{(K)}$, and we define the \emph{true top-$m$ set} for a given $m \in \{1,\dots,K-1\}$,
\[
    \mathcal{T}_m^\star \;=\; \bigl\{(1),\,(2),\,\dots,\,(m)\bigr\}.
\]
An algorithm with budget $n$ produces a \emph{ranked list} of $m$ models $\widehat{\pi} = (\pi_1,\pi_2,\dots,\pi_m)$, where $\pi_r$ denotes the index of the model predicted to occupy rank $r$, along with the corresponding predicted top-$m$ set $\widehat{\mathcal{T}}_m = \{\pi_1,\dots,\pi_m\}$.

\subsection{Ranking Procedure for Each Algorithm}
\label{sec:top-m-ranking-procedure}

Although the algorithms evaluated in \cref{sec:experiments} are designed to identify the best model, we introduce the following mechanisms to produce a ranked list for the models.

\paragraph{Mean-score Algorithms (US, SyUS, UCB-E, SyUCB-E).}
For algorithms that do not perform explicit arm elimination, we rank all $K$ models according to their empirical mean score upon the completion of draws:
\[
    \hat{\mu}_i \;=\; \frac{1}{|\mathcal{J}_i|} \sum_{j \in \mathcal{J}_i} S_{ij},
\]
where $\mathcal{J}_i \subseteq \mathcal{X}$ is the set of queries on which model $i$ was evaluated.
Models are sorted in decreasing order of $\hat{\mu}_i$ to obtain $\widehat{\pi}$.

\paragraph{Elimination-Based Algorithms (SR, SySRs).}
For algorithms that sequentially eliminate arms, we use the elimination order to assign ranks: the arm eliminated \emph{last} is assigned rank~$1$ (best), the arm eliminated second-to-last is assigned rank~$2$, and so on.
This is a natural byproduct of the algorithm's execution and requires no additional computation.

\subsection{Evaluation Metrics}

We evaluate top-$m$ performance using three complementary metrics. As in \cref{sec:experiments}, $b = \frac{n}{K_D \times L_D}$ denotes the budget percentage used for evaluation. Our experiments use a uniform grid of budget percentages with step size 5\%, from 0\% to 100\%.

\paragraph{Top-$m$ Identification Error Rate.}
For a single run, an algorithm \emph{fails to identify} the top-$m$ set if its predicted set differs from the true top-$m$ set:
\[
    \mathbf{1}\bigl[\widehat{\mathcal{T}}_m \neq \mathcal{T}_m^\star\bigr] \;\in \{0,1\}.
\]
For example, if the true top-3 set is $\{1,2,3\}$ and an algorithm predicts $\{2,1,3\}$, the result is $0$ (correct). If the algorithm instead predicts $\{1,2,4\}$, the result is $1$ (error).
The \emph{top-$m$ identification error rate} $\epsilon_b^{(m)} = \mathbb{P}(\widehat{\mathcal{T}}_m \neq \mathcal{T}_m^\star)$ is the probability of this event, estimated by averaging the binary outcome over independent runs.
At $m=1$ this coincides with the standard budget-indexed error probability $\epsilon_b$.

\paragraph{Top-$m$ Ranking Error Rate.}
We define the \emph{top-$m$ ranking error rate} as $0$ if and only if the predicted top-$m$ ranking exactly matches the true top-$m$ ranking,  and $1$ otherwise:
\[
    r_b^{(m)} \;=\; \mathbb{P}\bigl((\pi_1,\ldots,\pi_m)\neq\bigl((1),\ldots,(m)\bigr)\bigr).
\]
Since correct exact ranking implies correct top-$m$ set identification, we have $r_b^{(m)} \ge \epsilon_b^{(m)}$.

\paragraph{Top-$m$ per-rank error rate.}
The \emph{top-$m$ per-rank error rate} measures the expected fraction of incorrectly ranked models:
\[
    \rho_b^{(m)} \;=\; \mathbb{E}\!\left[\frac{1}{m}\sum_{r=1}^{m}\mathbf{1}\bigl[\pi_r \neq (r)\bigr]\right].
\]
For a single run, $\frac{1}{m}\sum_{r=1}^{m}\mathbf{1}[\pi_r\neq (r)] \in [0,1]$ expresses the fraction of top-m positions $r$ for which the predicted rank-$r$ model does not coincide with the true rank-$r$ model; we obtain an estimate $\hat{\rho}_b^{(m)}$ by averaging this quantity over independent runs. 

\paragraph{Computation of Evaluation Metrics for Tied Models}
Some of our datasets have multiple models tied for the same rank. In that case, we treat all produced rankings that are \emph{consistent} with the tied ranking as correct. For example, if two models are tied for rank two, ranking either of these models at rank two or three is counted as correct. Formally, we consider all proper rankings that are consistent with the tied ranking and for each run we compute the error rates as the minimum error rate over all these rankings.  

\paragraph{Top-$m$ Confidence Budget.}
Analogously to the X\% Confidence Budget (CB) defined in \cref{sec:experiments}, we define the \emph{X\% top-$m$ identification CB} as the minimum budget percentage $b_{D,m}^\star$ such that the empirical top-$m$ identification error rate $\hat{\epsilon}_b^{(m)}$ is at most $(100-X)\%$ for all evaluated budget percentages $b \geq b_{D,m}^\star$.
The \emph{X\% top-$m$ ranking-error CB} and \emph{X\% top-$m$ per-rank-error CB} are defined analogously. 
As in the best-arm case, we report the worst-case (maximum) CB across the 15 benchmark datasets.

\subsection{Results}
We present the following observations from the error-rate and CB views of the performance of algorithms on the Top-$m$ identification and ranking problem. 

\subsubsection{Top-$m$ Error-rate curves}
\Cref{fig:topm-identification-error-rate} reports $\hat{\epsilon}_b^{(m)}$ as a function of budget for $m\in\{3,5,10\}$, averaged across the 15 datasets, where lower curves indicate more sample-efficient recovery of the top-$m$ \emph{set}. \Cref{fig:topm-ranking-error-rate} shows the stricter ranking metric $\hat{r}_b^{(m)}$, which counts an error unless the full top-$m$ order is exactly correct. \Cref{fig:topm-per-rank-accuracy} provides the results for the alternative ranking metric \emph{per-rank accuracy error}, which measures the fraction of rank positions assigned incorrectly. 

There are similar trends across the results of the three metrics. SySRs is strongest for $m=3$ in terms of worst-case performance, and on par with SyUCB-E in terms of average error rates. For $m=5$, SySRs' worst case performance remains best, but SyUCB-E performs better in terms of average error rate, for the right hyperparameters. For $m=10$, this gap widens, and unsynchronized UCB-E starts to outperform SySRs in terms of worst case performance for larger target accuracies.  

\begin{figure}[h!]
\centering
\includegraphics[width=\textwidth]{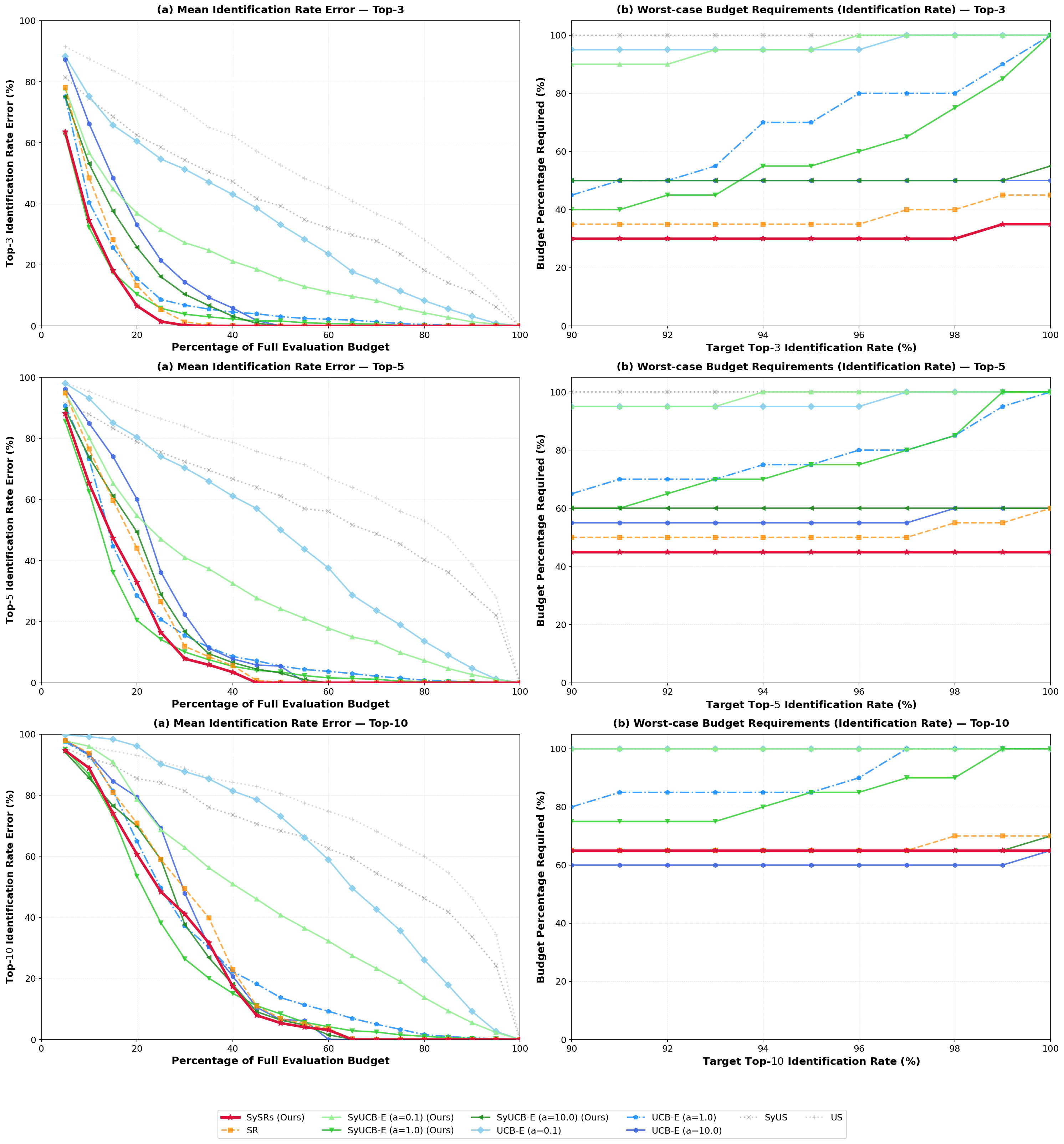}
\caption{Top-$m$ identification error rate $\hat{\epsilon}_b^{(m)}$ versus budget percentage, averaged over 15 datasets for $m\in\{3,5,10\}$. Lower is better.}
\label{fig:topm-identification-error-rate}
\end{figure}

\begin{figure}[h!]
\centering
\includegraphics[width=\textwidth]{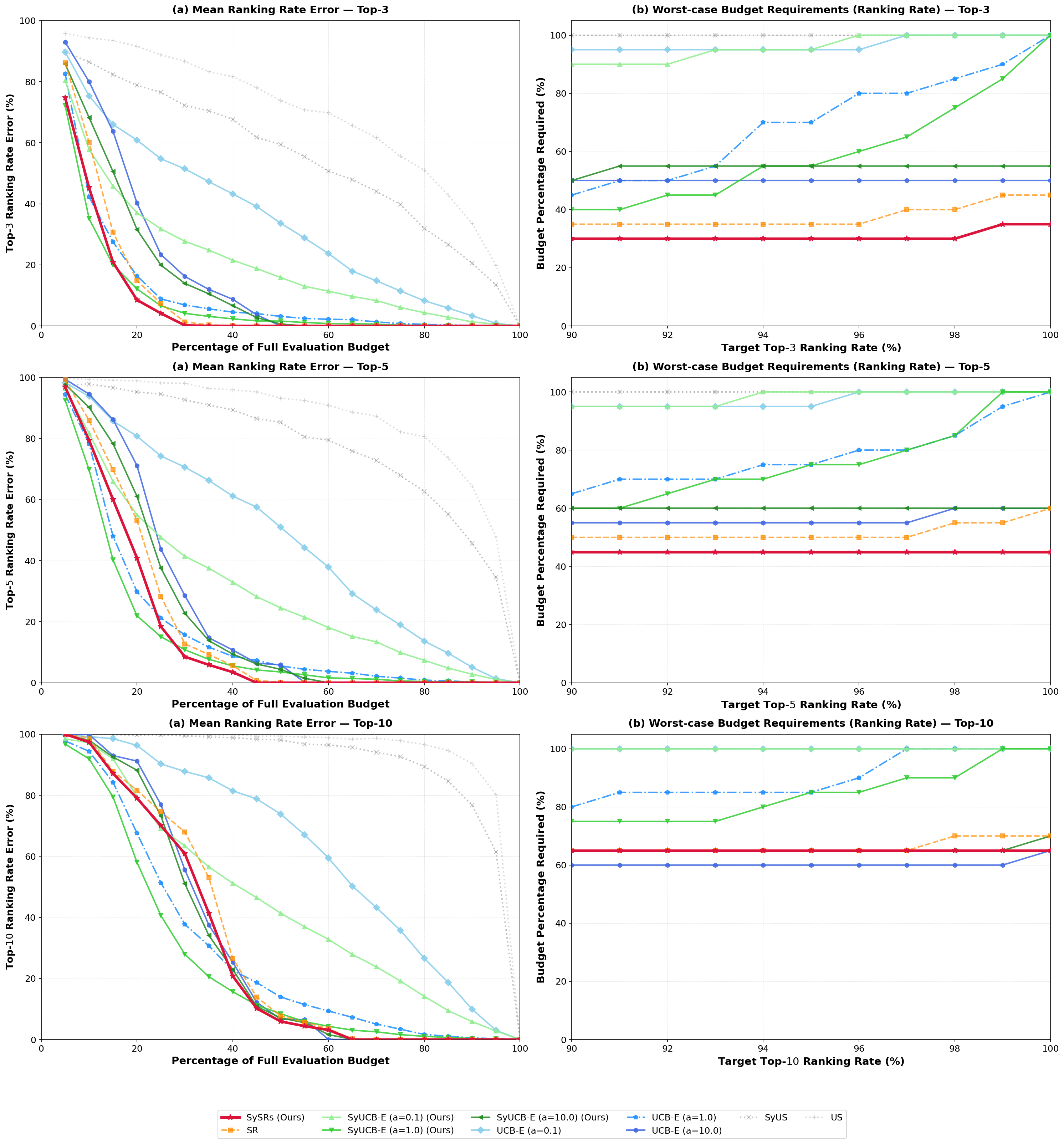}
\caption{Top-$m$ ranking error rate $\hat{r}_b^{(m)}$ versus budget percentage, averaged over 15 datasets for $m\in\{3,5,10\}$. Lower is better.}
\label{fig:topm-ranking-error-rate}
\end{figure}

\begin{figure}[h!]
\centering
\includegraphics[width=\textwidth]{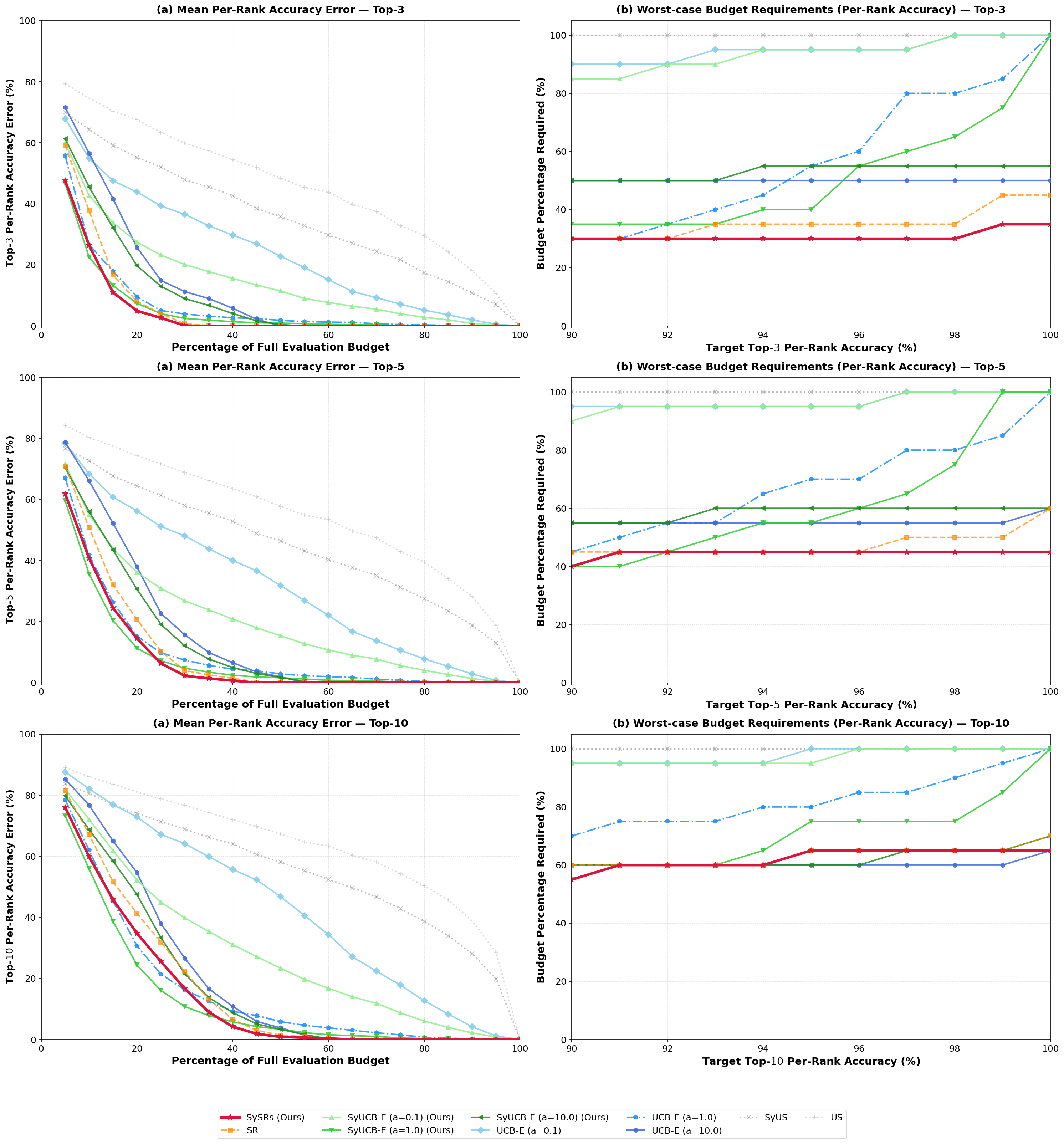}
\caption{Per-rank accuracy error versus budget percentage (left) and worst-case budget requirements for a given target per-rank accuracy (right), for $m\in\{3,5,10\}$, averaged over 15 datasets. Lower is better.}
\label{fig:topm-per-rank-accuracy}
\end{figure}}{}
\end{document}